\documentclass[lettersize,journal]{IEEEtran}
\usepackage{amsmath,amsfonts}
\usepackage{algorithm}
\usepackage{algorithmicx}
\usepackage{algpseudocode}
\usepackage{array}
\usepackage[caption=false,font=normalsize,labelfont=sf,textfont=sf]{subfig}
\usepackage{textcomp}
\usepackage{stfloats}
\usepackage{url}
\usepackage{verbatim}
\usepackage{graphicx}
\def\BibTeX{{\rm B\kern-.05em{\sc i\kern-.025em b}\kern-.08em
    T\kern-.1667em\lower.7ex\hbox{E}\kern-.125emX}}
\usepackage{balance}
\usepackage{paralist}
\usepackage{setspace}
\usepackage{enumitem}
\usepackage{amsthm}
\newtheorem{theorem}{Theorem}
\newtheorem{lemma}{Lemma}
\newtheoremstyle{bolddef}
  {3pt}   %
  {3pt}   %
  {\normalfont}  %
  {}      %
  {\bfseries} %
  {.}     %
  { }     %
  {}      %
\theoremstyle{bolddef}
\newtheorem{definition}{Definition}
\newcommand {\PI}[3]{\pi_{#1,#2}(#3)}
\newcommand{\bel}[3]{b_{#1,#2}(#3)}

\newcommand{\MU}[2]{\mu_{#1}(#2)}

\makeatletter
\def\munderbar#1{\underline{\sbox\tw@{$#1$}\dp\tw@\z@\box\tw@}}
\makeatother

\usepackage{cite} 
\usepackage{tikz}
\usetikzlibrary{arrows.meta, positioning}

\begin{document}
\title{ Structured Reinforcement Learning for \\ Bayesian Persuasion : Application to Intelligent Interactive Driving}

\author{Merlin Paul, \textit{Student Member, IEEE} and Anup Aprem, \textit{Member, IEEE}
}


\maketitle

\begin{abstract} 
Interactive driving, wherein an intelligent lead vehicle equipped with real-time traffic data coordinates route choices of connected vehicles,  offers a promising approach to dynamic traffic management.
To address the challenge of harmonising decisions — the lead vehicle prioritizes global traffic flow, while the connected vehicle seeks the shortest path, this paper considers the strategic information revealing framework of  Bayesian persuasion.  Here,  the principal (lead vehicle) aims to guide the agent’s (connected vehicle) partially observable sequential decision making towards its own objectives by selectively revealing information, such as real-time traffic ahead, using signals. However, the agent’s farsighted response to maximize its long-term reward, accounting for future state transitions and signal disclosures, renders the principal’s signaling strategy design computationally challenging. Moreover, in real-world settings, both the principal and agent may not have complete knowledge of their reward and transition dynamics. We propose an online structured reinforcement learning framework
to synthesize computationally efficient signaling strategy which is \emph{persuasive} for a far-sighted agent. In this context, the main contributions  of the paper are  as follows:
\begin{inparaenum}[(i)]
  \item For a monotonic agent with approximate best response, we  propose MAPL, a structured policy learning algorithm which utilizes the monotonic structure of the agent's policy for faster online learning in large state spaces,
\item Identification of sufficient conditions on the Bayesian persuasion model for the supermodular structure of the Q (action value) function of the principal for a monotonic agent, 
\item Identification of sufficient conditions to ensure the persuasiveness of the principal's signaling strategy for a monotonic agent, 
\item  Supermodular Q learning for Principal (SQP), which leverages the supermodular structure of principal's action value to synthesize computationally efficient  signaling strategy that is persuasive for a monotonic learning agent,
\item Numerical analysis considering a real-time application of Bayesian persuasive driving for lane selection demonstrates that the proposed method is  30\% cost efficient for optimising travelling rewards of both the lead and connected vehicle compared to the existing methodologies for signaling strategy design.
Moreover, online structured learning of connected vehicle's lane choices under Bayesian persuasion outperforms existing approaches for lane selection.
\end{inparaenum}
\end{abstract}
\section{Introduction} 
Intelligent traffic routing, which integrates real-time data into vehicle routing recommendations, has emerged as a promising approach to enhance adaptability and responsiveness to dynamic traffic challenges. 
Connected vehicles using an intelligent leader-based architecture offer a state-of-the-art solution for traffic optimization - refer to Fig. \ref{traff_rout} for an example. Here, the lead vehicle equipped with real-time traffic data can coordinate connected vehicles through speed modulation, route guidance, and platoon formation for efficient traffic management \cite{zhang2023cooperative}. This facilitates intelligent transportation networks, where the lead vehicle and the connected vehicle continuously collaborate through real-time interaction and learning. Conventional approaches for interactive driving face the critical challenge of harmonising decisions: — the connected vehicle seeks the shortest path, while the lead vehicle prioritizes global traffic flow and congestion avoidance, requiring coordination in real-time routing.  

 To address this challenge, this paper adopts the selective information disclosure framework of Bayesian persuasion to provide route guidance to individual vehicles while ensuring efficient traffic flow. Bayesian persuasion is a strategic framework in which a principal influences an agent’s actions by selectively revealing information through signals.
 This framework has found applications in several domains involving information asymmetry, including traffic routing \cite {gan2022bayesian}, recommendation systems \cite{mansour2022bayesian},  grading systems in educational settings \cite{boleslavsky2015grading}, queuing \cite{lingenbrink2019optimal}, etc.

In the sequential decision-making framework of Bayesian persuasion for intelligent interactive driving, illustrated in Fig.~1, both principal and agent observe the environmental state, while an additional hidden external state is known only to the principal. In the context of persuasive driving, this hidden state may represent real-time traffic conditions ahead. The principal strategically reveals information about the hidden state through signals—alerts about the traffic ahead or possible route selections—to influence the agent’s actions. A signaling strategy is a probabilistic mapping from the hidden external state and the agent’s state to a distribution over signals. A simple example is a direct revelation signaling strategy; however, directly revealing real-time traffic information may prompt many vehicles to choose the same route, ultimately causing congestion [8]. This underscores the need for carefully designed signaling strategies that guide routing decisions while maintaining efficient traffic flow.
 
Accordingly, the principal controls the flow of information by selectively disclosing these external states, aiming to align the agent's action with its objectives. The agent, on the other hand, faces a sequential decision-making problem and chooses actions based on its preferences.  A short-sighted agent maximizes the expected immediate reward, while a far-sighted agent takes into account the future state transitions and signal disclosures. As a result, the principal must effectively solve a nested planning problem — choosing signaling strategy while anticipating how the agent will respond to it.
In general, the design of signaling strategy while considering a far-sighted agent is NP-Hard \cite{gan2022bayesian}. 
Additionally, in real-world settings, both the principal and agent may
 not have complete knowledge of their reward and transition dynamics. Hence, we propose an online learning framework in which we model a reinforcement learning agent that approximately best responds.
\begin{figure}
\centering
\includegraphics[width=9.5cm,height=5.5cm]{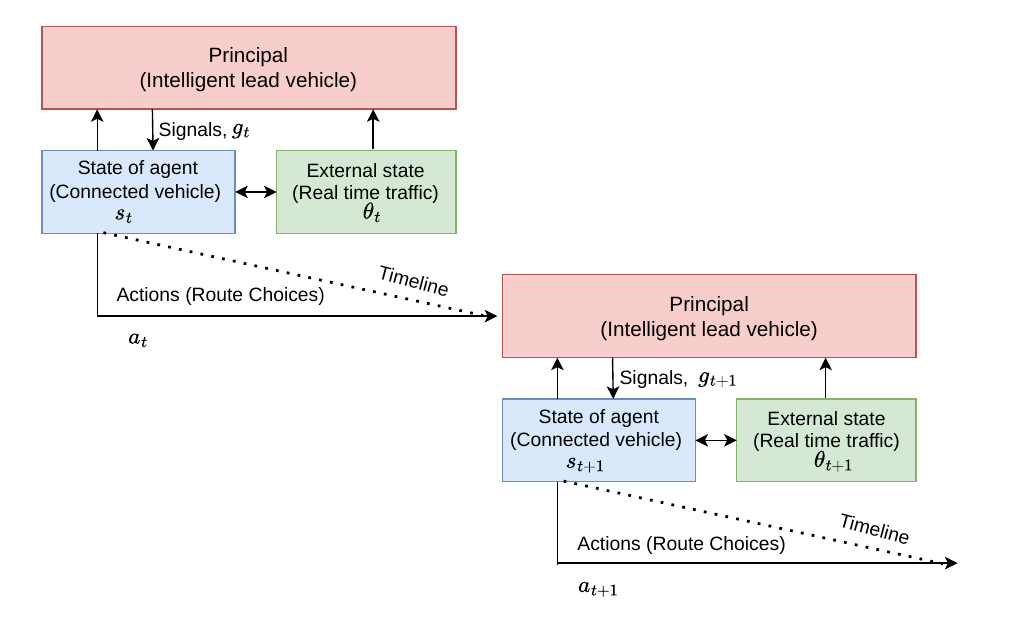}
\caption{ The schematic representation of intelligent interactive driving based on Bayesian persuasion is illustrated in Section~\ref{Prob formln}. Here,
both the principal (intelligent lead vehicle) and the agent (connected vehicle) are aware of the vehicle's state, $s$. An external state, $\theta$ (real-time traffic), known to the principal but unknown to the agent, influences the decision outcome. The principal uses this information to persuade the agent through a signal $g$, guiding actions toward its objectives. Given the signaling strategy, the agent acts on its posterior belief about $\theta$, incurring an immediate reward and inducing a stochastic transition to the next state.
}
 \label{Bayesian persuasion}
\end{figure}

This paper considers a monotonic agent - an agent that responds to `higher' signals by choosing `higher' actions, inducing an ordered and predictable response to signal disclosures. Such behaviour is natural in many practical settings— e.g., transmission control, sensor scheduling, traffic routing — where stronger signals lead agents to better actions \cite{ma2025line}, \cite{paulmonotonic}.
 Many real-world applications—such as credit rating, queuing, traffic routing — where Bayesian persuasion is widely used, often exhibit monotonic threshold policies, such as assigning higher ranks to higher-quality firms \cite{goldstein2018stress} or increasing travelling speed as the road traffic improves \cite{paulmonotonic}. 

We propose an online structured reinforcement learning framework
to synthesize computationally efficient signaling strategy that is
\emph{persuasive} for a far-sighted agent. The  main contributions of the paper are:
\begin{enumerate}
\item \textbf{Monotonic Agent Policy Learning (MAPL)}:    We assume an agent with  monotonic policy that approximately best responds and propose a structured reinforcement learning algorithm that utilizes the monotonic structure of the agent policy. Algorithm 1 (MAPL) enables faster online learning of agent policy in large state spaces while achieving higher average reward.
\item \textbf{Characterisation of principal's Q factors:} Theorem~1 gives sufficient conditions on the  Bayesian persuasion model to ensure supermodular  Q value function of the principal, considering a monotonic agent. 
\item \textbf{Structural results for persuasive signaling strategy:} 
We present sufficient conditions that ensure the persuasiveness of the signaling strategy for a monotonic agent. Specifically, Theorem 2 characterizes the principal’s signaling strategy, while Theorem~3 derives sufficient conditions on the agent’s reward function which ensure that the designed signaling strategy is persuasive for a monotonic agent. 
\item \textbf{Supermodular Q learning for Principal (SQP)}:  We propose Algorithm 2 (SQP), which leverages the supermodular structure of the principal’s Q-value to synthesize computationally efficient signaling strategy which is persuasive for a monotonic agent. Furthermore, Algorithm~3 presents the online reinforcement learning framework for the
repeated interaction of
the principal with  approximately best responding agent in dynamic Bayesian persuasion. 

\item \textbf{Case study of Bayesian persuasive driving }: We illustrate 
\begin{inparaenum}[(i)]   
    \item  the computational efficiency of the  online   learning  of the connected vehicle's policy using MAPL compared to conventional dynamic vehicle routing recommendations,
    \item cost effectiveness as well as computational efficiency of the proposed structured reinforcement learning framework in the design of signaling strategy for the lead vehicle's intelligent interactive driving against existing methodologies \cite{wu2022sequential},\cite{gan2022bayesian}, \cite{bernasconi2022sequential}. 
   \end{inparaenum} 
\end{enumerate}
\emph{To the best of our knowledge, this is the first work that explores
the structural results for online reinforcement learning in dynamic Bayesian
persuasion framework, with applications in intelligent traffic routing, transmission control, credit rating,  sensor scheduling, and queuing.}

\section{Related Work}
Interactive driving approaches that anticipate and adapt to surrounding vehicle's lane-change behaviours have gained increasing attention in intelligent traffic management~\cite{peng2019bayesian}. In this work, we focus on \emph{lane-based routing}, where decisions are made at the granularity of individual lanes, enabled by advances in vehicle-to-vehicle (V2V) communication, real-time traffic sensing, and precize localization. This fine-grained control facilitates coordinated lane assignment and platoon formation, allowing a well-informed lead vehicle with a global view of traffic conditions to guide connected vehicles into high-speed formations. Such coordination improves traffic throughput, enhances energy efficiency through reduced aerodynamic drag, and promotes road safety via synchronized vehicle behavior~\cite{zhang2023cooperative}.

\paragraph{Existing lane-changing approaches}
Classical lane change approaches for autonomous vehicles rely primarily on rule-based decisions that evaluate current traffic conditions against predefined safety constraints, such as minimum gap distance and time-to-collision thresholds \cite{yang2019examining}. However, these methods depend on the instantaneous traffic state rather than anticipating future traffic evolution.  Model Predictive Control (MPC) addresses this limitation by optimising vehicle trajectories over a finite prediction horizon while satisfying the dynamic and safety constraints \cite{zhao2025survey}. However, its performance is highly sensitive to model accuracy and becomes computationally demanding in complex traffic environments, limiting its practicality for real-time deployment. 

\paragraph{Online reinforcement learning approaches}  Deep neural networks have been used to model complex lane-changing behaviour  \cite{wei2021driver}, \cite{zhao2025survey}. However, these methods typically require large-scale training data and substantial computational resources \cite{zhao2025survey}. Furthermore, recent advances have shifted toward online reinforcement learning methods that learn lane choices through repeated environmental interactions \cite{alizadeh2019automated}. However, in multi-vehicle settings, fully shared information can paradoxically lead to traffic congestion \cite{liu2019efficient}. This has motivated the exploration of  strategic information revealing framework, with Bayesian persuasion emerging as a promising approach for intelligent traffic management
\cite{peng2019bayesian}.

\paragraph{Bayesian persuasive driving}Bayesian persuasion provides a strategic framework in which a lead vehicle (the principal) selectively reveals real-time traffic information to influence the lane selections of a connected vehicle (the agent). In the context of Bayesian persuasive driving, it is typically assumed that the principal has complete prior knowledge of the environment, allowing signaling strategies to be precomputed \cite{gan2022bayesian}.
 Moreover, the signaling strategy of the lead vehicle as well as the posterior belief of the connected vehicle are assumed to be  Gaussian distributed for computational tractability. In addition, a simpler agent model is assumed without stochastic state transitions. In this paper, we consider the online reinforcement learning setting, in which the principal as well as the agent are unaware of the reward, the state transition dynamics, and the prior belief over external states, which makes the design of signaling strategy challenging.

\paragraph{Myopic agent assumption} Furthermore, for computational tractability of design of signaling strategy, numerous works restrict attention to myopic agent, which optimizes its immediate rewards without considering future transitions and signaling disclosures. 
For instance, \cite{gan2022bayesian} designs signaling strategies under a fully known model with a myopic agent assumption. More recently, \cite{wu2022sequential} extends this line of work to an online reinforcement setting, where the principal must simultaneously learn the prior distribution, reward function, and transition dynamics. 
However, there is growing interest toward dynamic Bayesian persuasion, in which the principal and agent interact repeatedly over time \cite{gan2022bayesian}, \cite{liu2019efficient}, \cite{ely2017beeps}, which we consider in this paper.

\paragraph{Dynamic Bayesian Persuasion}Determining the optimal signaling strategy for dynamic persuasion, assuming the optimum policy of the agent is generally NP-Hard \cite{gan2022bayesian},\cite{wu2022sequential}. To address farsightedness, under the assumption of model dynamics,  \cite{bernasconi2023persuading} employs history-dependent strategies like promise form schemes, which ensure future reward promises. \cite{bernasconi2025online} explores
online learning framework in which the principal, without any knowledge of the prior belief of the agent about the hidden external state, repeatedly interacts with the far-sighted agent and gradually learns the persuasive signaling strategy. Moreover,  it assumes that all other underlying model parameters are known and the agent possesses perfect recall of memory. This paper considers a far-sighted agent and develops computationally and memory-efficient signaling strategy design as compared to \cite{bernasconi2025online}.

\paragraph{Online Bayesian persuasion for large state spaces}Online learning of signaling strategy in large state, external state, action, and signal spaces is often computationally intractable due to the vast exploration domain. Existing literature on dynamic Bayesian persuasion in continuous spaces considers a simplified model with myopic agent \cite{peng2019bayesian},  linear reward, and transition dynamics \cite{wu2022sequential} assumptions. However, we consider a far-sighted agent under general reward and transition dynamics. 
In the context of dynamic Bayesian persuasion, we explore structural results for faster online learning of both agent's and principal's strategies in large state spaces.

\paragraph{Structural results}Many real world applications naturally satisfy sufficient conditions that guarantee the optimality of monotonic policies \cite{krishnamurthy2018multiple}, \cite{wu2020optimal}, \cite{aprem2013transmit}. In the context of Bayesian persuasion, monotonic signaling—where higher actions are recommended under higher external states—has found applications in diverse areas such as credit rating and traffic routing. In such settings, agents often exhibit threshold policies, such as assigning higher ranks \cite{goldstein2018stress} to higher-quality firms or increasing speed as the road traffic improves \cite{ ngo2009optimality}. Similarly, \cite{djonin2007q} utilizes supermodularity as a linear constraint on Q factors to improve learning. Additionally, \cite{dughmi2016algorithmic} characterizes signaling strategy when the principal's utility function is supermodular.  However, these works restrict the interaction between the principal and agent in one-shot settings. 
     
\paragraph{Novelty of our work} This paper utilizes structural results for online reinforcement learning in dynamic Bayesian persuasion. Our approach provides a joint characterisation of both the agent’s policy and the principal’s value function. Additionally, structured learning ensures faster convergence   \cite{brams2019farsightedness} in high-dimensional or model-free dynamic environments. In particular, the monotonic characterisation of the agent’s policy substantially reduces computational complexity and memory requirements by restricting the search space to policies that satisfy the monotone structure. This enables faster online algorithms for computing the agent’s optimal policy. Furthermore, by exploiting the supermodularity of the principal’s Q-value function, we develop a structured online reinforcement learning framework that synthesizes computationally efficient signaling strategy while ensuring persuasiveness for a monotonic agent.

\subsection*{Outline }
Section \ref{Prob formln} presents the Bayesian persuasion framework and formulates the Bellman dynamic programming equation to compute the optimal policy of the agent and principal's signaling strategy.
Section \ref{Structural Results} presents the main result, Theorem~\ref{Therm main}, which gives the sufficient conditions for the supermodular structure of the Q value function of the principal for a farsighted monotonic agent. Additionally, Theorem \ref{Thm-IC} characterizes the structural results on the principal’s signaling strategy, while Theorem \ref{Theorem agent} derives the sufficient conditions on the agent’s reward function that ensure that the designed signaling strategy is persuasive for a monotonic agent. Section~\ref{Structured  Agent's learning} presents a structured reinforcement learning framework with computationally efficient algorithms, MAPL for structured learning of the agent's monotonic policy, and SQP for supermodular Q learning for the principal's action value, respectively. Section \ref{Numerical results} illustrates the practical applicability of the proposed model assumptions for Bayesian persuasive driving. The numerical results highlight the usefulness of structural results in enhancing computational efficiency in learning the optimal policy of the connected vehicle as well as the signaling strategy of the intelligent lead vehicle. Concluding remarks are offered in Section~\ref{Conclusion}. The Appendix at the end contains supporting lemmas, assumptions, proofs of algorithms, and theorems.

\section{Problem formulation} \label{Prob formln}
  \begin{figure}[htbp]
    \centering  \includegraphics[width=8.5cm,height=6cm]{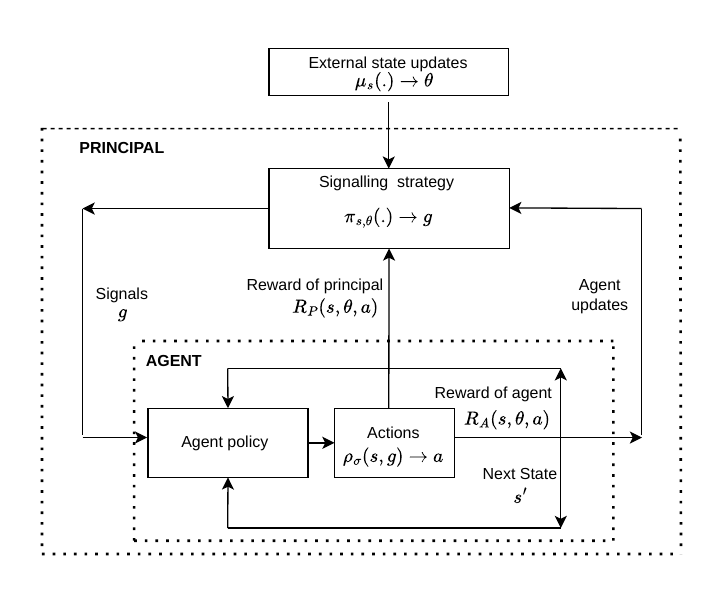}
    \caption{Dynamic framework of Bayesian persuasion in which the principal guides the sequential decision making of the agent by revealing the unknown external states through strategic signaling.}
    \label{MPP}
\end{figure}

 Fig. \ref{MPP} illustrates the dynamic framework of Bayesian persuasion in which the principal guides the sequential decision making of the agent by revealing the unknown external states using a signaling strategy. At each time step, the principal observes the external state of the environment and samples a signal according to the committed signaling strategy. Upon receiving the signal, the agent updates its posterior belief about the unknown state and selects an action by maximizing its expected reward. This accounts for both immediate payoff and state transitions, which are observed by the principal and agent. Subsequently, a new external state of the environment evolves, which is known only to the principal.  This dynamic repeated interaction between the principal and agent continues over time. The goal of the agent is to learn to choose actions that maximize its expected average reward under the committed signaling strategy of the principal. On the other hand, the objective of the principal is to design signaling strategy to maximize its average reward while ensuring that the persuaded learning agent also benefits from the signal disclosures.
\subsection{System Model}
The model can be represented  by the tuple $ M~ = ~\langle \  \mathcal{S},\Theta,\mathcal{G}, \mathcal{A}, P, (\mu_s), (\pi_{s,\theta}), R_A, R_P \  \rangle $
\begin{itemize}
  \item \textbf{State ($s \in \mathcal{S} = \{0,1,2 ..S\}$):} State refers to the known environment of the agent. The state of the agent is also observable by the principal. In the context of traffic routing, it represents measurable factors such as vehicle’s location, velocity, acceleration, or energy consumption~\cite{kamrani2020applying}.
  \item \textbf{External State ($\theta \in \Theta = \{0,1,2..O\}$):} It denotes hidden factors which are not directly known to the agent but observed by the principal, such as road accidents, construction works, which have a significant impact on the agent's decision outcomes.
  \item \textbf{Signal ($g \in \mathcal{G} = \{1, 2, \dots, G\}$):} The signal space denotes the set of signals, which the principal sends to the agent based on the observed state and external state. In the context of traffic routing, signaling may be binary, discrete, or continuous—for example, indicating whether to switch lanes, alerting about upcoming congestion, recommending an appropriate driving speed, or specifying the current traffic condition \cite{peng2019bayesian}.
  \item \textbf{Action ($a \in \mathcal{A} = \{1, 2, \dots, A\}$):} Action space refers to the set of possible actions that the agent can choose based on the signal received. For example, in traffic routing, accelerating, decelerating, selection of different lanes, etc., are taken as actions  \cite{kamrani2020applying}.
  \item \textbf{Prior belief ($\mu(s,\theta) = P(\theta \mid s)$):} It is the conditional probability density  of the external state $\theta$ given the state $s$. 
  It serves as the prior belief distribution of the external state $\theta$ in state $s$, known to the agent before any interaction with the principal, and captures the correlation between state and external state.
  \item \textbf{Signaling strategy ($\PI{s}{\theta}{g} = P(g \mid s, \theta)$):} It defines the conditional probability distribution of signals used by the principal, based on the joint realization of $s$ and $\theta$.  In our model, the signaling strategy is assumed to be Markovian \cite{gan2022bayesian} in which the principal samples a signal $g$ based on the current state $s$ and external state $\theta$.
  \item \textbf{Reward ($R_A(s,\theta, a), R_P(s,\theta, a)$):} It denotes the immediate reward to the agent and principal,  respectively,  determined by the agent’s action as well as state and external state.
  \item \textbf{State transition ($P(s' \mid s, a)$):} It specifies the stationary probability distribution over the next state $s'$,  given  the current state $s$ and the action $a$ taken by the agent.
\end{itemize}
Classical MDP settings involve an agent that either knows or learns the environment through direct interaction, whereas in Bayesian persuasion, the agent’s action is influenced by the principal’s signaling strategy.
\subsection{Agent's average reward policy under Bayesian persuasion}
\paragraph*{Agent's belief over the external state, $\theta$}
The agent does not have direct knowledge of the external state $\theta$. 
Upon observing a signal $g$, while being in state $s$, the agent updates its belief about $\theta$ using Bayes’ rule, resulting in the posterior belief, $\bel{s}{g}{\theta}$, given by
 \begin{align}
\label{eq:belief_update}
\bel{s}{g}{\theta} 
&= P(\theta \mid s, g) 
   = \frac{\PI{s}{\theta}{g} \cdot \MU{s}{\theta}}
          {\sum_{\theta \in \Theta} \PI{s}{\theta}{g} \cdot \MU{s}{\theta} }
\end{align} 
 At each instant $t$, based on the signaling strategy, the principal samples signal $g_t$, depending on the current state $s_t$, and external state $\theta_t$. 
 Since the external state is unobserved by the agent, the expected immediate reward, $\tilde{r}_A$, obtained on taking action, $a_t$, can be computed by taking the expectation with respect to the posterior belief, $b_{s_t,g_t}(\theta_t)$, which is given by  
\begin{align}
\label{imm_reward}  
\tilde{r}_A(s_t,g_t, a_t)=\sum_{\theta \in \Theta} \bel{s_t}{g_t}{\theta_t}\cdot R_A(s_t,\theta_t, a_t) \end{align}
A policy, $\rho: \mathcal{S} \times \mathcal{G} \rightarrow \mathcal{A} $, is the decision rule that determines the action $a$ to be taken at a particular state $s$ on receiving a signal $g$. 
 In this paper, the objective of the agent is to  maximize expected average reward from the starting state $s$ with signal $g$  expressed as 
\begin{align}
\label{agent_exp_cum_cost}
\tilde{R}_A\!= \lim_{t \xrightarrow{}\infty}\frac{1}{t} \mathbb{E} \biggl [\sum_{t=0}^{\infty} \tilde{r}_A(s_t,g_t, a_t) | s_0\!=\!s,g_0\!=\!g \biggl]
\end{align}
Under the ergodic Markov chain assumption of the agent's policy, this limit exists and is independent of the initial state $(s,g)$ \cite{puterman1990markov}. Here, the expectation is taken with respect to state transition dynamics, signaling strategy, and prior belief distribution. 
 Agent's action $a$ in state $s$ results in an expected immediate reward as well as a stochastic transition to a new state ~$s'$. The new external state is given by $\theta'\sim~ \mu_{s'}(.)$.
The principal who is aware of the realisation of $\theta' $, samples signal $g' \sim \pi_{s',\theta'}(.)$ and transmits it to the agent. This  transition, $(s, g)$ to $(s', g')$, happens  with  probability
\begin{align*}
P((s,g),(s',g'),a)=P(s,s',a)\cdot \sum_{\theta' \in \Theta} \MU{s'}{\theta'}\cdot \PI{s'}{\theta'}{g'} 
\end{align*}
$V_A(s,g)$ is the differential or relative value function of the agent at state $s$ with signal $g$.
It denotes the reward obtained relative to the average reward while following policy $\rho$ starting from the state--signal pair $(s,g)$. The corresponding action value function $Q_A(s,g,a)$, represents the expected relative reward of taking action $a$ in $(s,g)$ under policy $\rho$, and can be considered as the sum of the expected immediate reward in state $s$ and the expected value of the subsequent state--signal pair $(s',g')$.  The Bellman equation for the agent's policy  is given by 
\begin{align}
\label{Q_contin_a}
& \tilde{R}_A + Q_A(s,g,a) =  \tilde{r}_A(s, a) +\\ &\sum_{s' \in \mathcal{S}} P(s' \mid s, a) \sum_{g' \in \mathcal{G}} \sum_{\theta' \in \Theta}  \MU{s'} { \theta' } \cdot \PI{s'}{\theta'}{g'} \cdot V_A(s',g') \notag
\end{align}
 Let $Q_A^*, V_A^*, \rho^*, R_A^*$ be the optimal action value, value function, policy, and average reward of the agent.  The Bellman optimality equation can be written as
\begin{align}
\label{Q_contin_b}
 &R_A^* + Q_A^*(s,g,a) =  \tilde{r}_A(s, a) +\\ &\sum_{s' \in \mathcal{S}} P(s' \mid s, a) \sum_{g' \in \mathcal{G}} \sum_{\theta' \in \Theta}  \MU{s'} { \theta' } \cdot \PI{s'}{\theta'}{g'} \cdot\max_{a \in A} Q_A^*(s',g',a) \notag
\end{align}
\begin{align}\label{opt_agent_policy}
&\rho^*(s,g) =\arg \max_{a \in \mathcal{A}} Q_A^*(s,g,a) \\ \notag 
& V_A^*(s,g) =\max_{a \in \mathcal{A}} Q_A^*(s,g,a) \notag
\end{align}
In the next section, we demonstrate how this agent's policy is taken into account in the principal's signaling strategy design.

\subsection{Signaling strategy of the principal for Bayesian persuasion
}\label{signaling strategy of the principal}

Bayesian persuasion differs from the classical MDP in the fact that the actions are  not taken directly by the principal, but by the agent under the influence of the information shared. The goal of the principal is to strategically reveal the unknown external state to guide the agent's action to align with its objectives. In this paper, we consider a Markovian signaling strategy that maps the external hidden state for each of the agent's state to a probability distribution of signals,  i.e., $\pi: \mathcal{S} \times \Theta \rightarrow \Delta \mathcal{G}.$ Moreover, in the context of reinforcement learning, where both the principal and agent learn together, probabilistic recommendations reduce the risk that signals are completely ignored by the agent \cite{lin2023information}. At each step, given the state \(s\) and external state \(\theta\), the principal samples a  signal,  $g \sim \pi_{s,\theta}(.)$.  The objective of the principal is to select a signaling strategy  which maximizes its long run average reward  from a starting state ($s$, $\theta$) while taking into account the agent's  action given by  (\ref{opt_agent_policy}) expressed as
 \begin{align}
\label{prin_exp_cum_reward}
\tilde{R}_P=\lim_{t \xrightarrow{}\infty}\frac{1}{t} \mathbb{E}_\pi \biggl[\sum_{t=0}^{\infty}  R_P(s_t,\theta_t, a_t) | s_0\!=\!s,\theta_0\!=\!\theta \biggl]
\end{align}
  
 Here, the expectation is taken with respect to the agent's policy, state transition dynamics, signaling strategy, and prior belief distribution. Let $V_P(s,\theta)$ be the differential value function of the principal in state $s$ associated with $\theta$, which gives the reward obtained relative to the average reward $\tilde{R}_P$, under signaling strategy $\pi$ starting from $(s,\theta)$. 
 The transition probability from $(s,\theta)$ to $(s',\theta')$ is given by:
\begin{align*}
P((s,\theta),(s',\theta'),a')=P(s,s',a')\cdot \MU{s'}{\theta'}\end{align*} $Q_P(s,\theta,g)$ is  the signal value of the principal in $(s,\theta)$ on sending a signal $g$ under the signaling strategy $\pi.$ It can be expressed as the sum of the principal's immediate reward based on the agent's action prompted by the signal $g$ in $(s,\theta)$ and the expected value on the transition to the future state $(s',\theta')$.

We consider a learning agent that approximately best responds \cite{chen2023persuading}, selecting stochastic actions that are near-optimal with high probability during learning. Let   $\tilde{\rho}(a|s,g)$ be the probability that the learning agent takes an action $a$, on receiving a signal $g.$
The Q value function of the principal with an average reward $\tilde{R}_P$ can be formulated as
\begin{align}
\label{Principal Q}
&\tilde{R}_P + Q_P(s,\theta,g)=\sum_{a' \in \mathcal{A}} \tilde{\rho}(a'|s,g) \cdot R_P(s,\theta,a')+ \\  &\sum_{a'\in \mathcal{A}} \tilde{\rho}(a'|s,g) \cdot \sum_{s' \in \mathcal{S}} P(s,s',a')  \sum_{\theta' \in \Theta}  \mu_{s'}(\theta')  \cdot V_P(s', \theta') \notag \\
&V_P(s,\theta)
=\sum_{g \in \mathcal{G}} Q_P(s,\theta,g) \cdot \pi_{s,\theta}(g) \notag
\end{align}
\subsection{Persuasive signaling : Incentive compatibility of the agent}\label{Persuasive signaling strategy}
The principal sends a signal that selectively discloses the privately observed external state, anticipating that the agent will take the signal into account and act in a way beneficial to the principal. A simple signaling strategy is to reveal the true external state. However, the direct revelation of external state, for example, real-time traffic, may prompt connected vehicles to choose the same route and may finally lead to traffic congestion \cite{liu2019efficient}. Moreover, a key challenge in information design is that such communication can be effective only if the agent finds it profitable to follow the signal. Hence, the signal cannot serve the principal’s interest solely; it must also improve the agent’s expected outcome so that adhering to it is rational from the agent’s perspective. 
\paragraph*{Signaling strategy as action recommendation}
Without loss of generality,  a signaling strategy can restrict the number of signals equal to the receiver’s action set \cite{gan2022bayesian}, i.e., $G_A =\{g_a: a \in \mathcal{A} \}$.  
Let $Q_A(s,\theta,g_a)$ be the action value of the agent in $(s,\theta)$ while recommended by a signal $g_a$, for taking an action $a$, under the signaling strategy $\pi.$
Thus, for the persuaded agent to benefit from strategic signaling of the principal,  its expected average reward for the recommended action has to be greater as compared to other actions, i.e., 
\begin{align}\label{IC}
&Q_A(s,g_a,a)\geq Q_A(s,g_a,a'), \forall a,a' \in \mathcal{A} 
\end{align}
Hence, in order to ensure a persuasive signaling strategy, the incentive compatibility of the agent is taken into account while maximizing the principal's expected average reward. Let $\beta(s,\theta)$ be the probability distribution of the initial state being $(s,\theta)$. The objective of the principal is to maximize the average reward given by (\ref{prin_exp_cum_reward}) subject to the constraints in (\ref{Principal Q})  and (\ref{IC}).
Thus, the optimization of the principal can be formulated \cite{krishnamurthy2016partially} as 
\begin{align}
\label{LP1}
&{\min_{V_P}}\sum_{s \in \mathcal{S}}\sum_{\theta\in \Theta}
\beta(s,\theta) \cdot  V_P(s,\theta)
\end{align}
\textnormal{subject to }
\begin{align*}
&\tilde{R}_P + Q_P(s,\theta,g_a)\geq \sum_{a' \in \mathcal{A}} \tilde{\rho}(a'|s,g_a) \cdot R_P(s,\theta,a')+ \\  &\sum_{a' \in \mathcal{A}} \tilde{\rho}(a'|s,g_a) \cdot \sum_{s' \in \mathcal{S}} P(s,s',a')  \sum_{\theta' \in \Theta}  \mu_{s'}(\theta')  \cdot V_P(s', \theta') \notag 
\end{align*}
\begin{align*}
&Q_A(s,g_a,a)\geq Q_A(s,g_a,a'), \forall a,a' \in \mathcal{A}  
\end{align*}
When the agent is far-sighted, its optimal action does not only depend on the expected immediate reward but also on future state transitions and signal disclosures.  Consequently, the principal’s value function depends on the signaling strategy as well as the induced agent policy, which is the solution of a dynamic optimisation as illustrated in (\ref{Q_contin_b}). This nested dependence makes the principal’s optimization nonlinear, and the design of signaling strategy NP-Hard \cite{gan2022bayesian}.
 Hence, in the next section, we explore the structural results to compute efficiently the optimal agent policy as well as the signaling strategy.
\section{Structural Results}\label{Structural Results}
Here, we investigate how structural results on the principal can be leveraged for the design of a persuasive signaling strategy for a far-sighted agent. We restrict attention to a monotone agent, which enables us to use the Monotone Likelihood Ratio (MLR) ordered signaling strategy, which is \emph{persuasive}, as illustrated in the following subsection.
\subsection{Monotonic agent policy}
 The optimal policy of the agent in a Bayesian persuasion environment often exhibits monotone behaviour with respect to the state and signal space. This is widely observed in various real-world applications, including traffic routing \cite{paulmonotonic}, sensor scheduling \cite{krishnamurthy2016partially}, inventory management~\cite{puterman1990markov}, transmission scheduling \cite{aprem2013transmit}, etc. The monotonic policy can be characterized as  
\begin{align}
    \rho(s,g) \leq \rho(s',g), \   s \leq s', \ \forall  s,s' \in \mathcal{S},\ g \in \mathcal{G}, \\
    \rho(s,g) \leq \rho(s,g'), \   g \leq g', \  \forall s \in \mathcal{S},\ g, g' \in \mathcal{G}. \notag
\end{align}
Subsequently,  in the context of Bayesian persuasion,  assuming monotonic agent policy, we characterize the structure of the principal's Q value function for a persuasive signaling strategy as detailed below.
\subsection{Supermodularity of Q function and Incentive compatible signaling strategy design:}
\label{Structural Q }
In this section, we explore structural properties of principal's Q-functions to ensure that the induced signaling strategy is incentive compatible for a monotonic agent. The main result is Theorem \ref{Therm main}, which gives sufficient conditions for supermodularity of the Q value of the principal with respect to the state and external state. Consider the following assumptions on the Bayesian persuasion model:
\begin{enumerate}
\item[(A1)] $R_P(s,\theta,a)$ is increasing in the state of agent $s$,  external state $\theta$, action $a$ taken by the agent.
\item[(A2)] $R_P(s,\theta, a)$ is supermodular in $(s, a)$ for fixed $\theta$ and $R_P(s,\theta,a)$ is supermodular in $(\theta, a)$ for fixed $s$.
\item[(A3)]  $\mu_{{s_i}}(\theta)  \leq_{r} \mu_{{s_j}}(\theta), \   \forall i < j,$ i.e., prior belief of the external states of higher states MLR  dominates the lower.
\item[(A4)] $P (s_i,., a) \leq_s  P(s_j,., a), \  \forall s \in \mathcal{S}, a \in \mathcal{A}, i \leq j$, i.e., each state of the transition probability stochastically dominates the previous state. 
\item[(A5)] $P(s,s',a+1)$ is tail sum supermodular in $s,a$ in the sense  
$\sum_ {\bar{s}}^S ( P(s,s',a+1)- P(s,s',a))$    is increasing in~ $s, \ \forall  \  \bar{s} \in \mathcal{S}$.
\item[(A6)] The learning agent's policy, $ \tilde{\rho}(a|s,g)$ stochastically 
dominates the previous state and signal spaces and is tail sum supermodular in $s,g$.
\end{enumerate}
\begin{theorem}\label{Therm main}
Assuming that the  Bayesian persuasion model satisfies (A1) to (A6), then there exists an optimal principal's action value function which is supermodular in state and external space with respect to signals. 
\end{theorem}
The discussion of the assumptions is given in the Appendix~\ref{Assumptions}. 
In the following section, we investigate how the signaling strategy of the principal can be made incentive compatible for a monotonic agent and state it as follows:
\begin{theorem}\label{Thm-IC}
 When the principal commits to an MLR-ordered signaling strategy, it ensures monotone, incentive-compatible agent behaviour.
\end{theorem}
MLR signaling results in the agent's monotone belief about external states. Since the agent updates posterior belief of unknown external state,  based on the signal transmitted by the principal using Bayes' rule—and MLR is preserved under Bayesian updating—the principal must commit to an MLR-ordered signaling strategy to guarantee monotone, incentive-compatible agent behaviour. This is proved in detail in the Appendix \ref{app:Monotone best response}.
Next, we illustrate how the supermodular structure of the principal's Q value guarantees monotone, incentive-compatible agent behaviour in the following lemma.
\begin{lemma}
Softmax strategy over the supermodular Q values of the principal ensures MLR ordered signaling strategy.
\end{lemma}
We compute the signaling strategy as the softmax policy  considered for  the average reward setting as 
\begin{align}\label{bolt_signaling}
 \pi(g | s,\theta) = \frac{\exp\left( Q(s,\theta,g) / \tau \right)}{\sum_{g' \in \mathcal{G}} \exp\left( Q(s,\theta, g') / \tau \right)}
\end{align}
where $\tau$ is the temperature parameter controlling exploration. It is proved in the Appendix \ref{lem:softmax-mlr} that the softmax strategy preserves MLR.
\subsection{Structural results for agent's policy learning }
In addition to  MLR ordered signaling strategies committed by the principal, (A3) and (A4) satisfied by the Bayesian persuasion model, the following assumptions on the reward structure of the agent are required to ensure a monotone best response of the agent.
 \begin{enumerate}
 \item[(B1)] $R_A(s,\theta,a)$ is increasing in state of agent $s$,  external state~$\theta$.
 \item[(B2)]  $R_A(s,\theta, a)$ is supermodular in $(s, a)$ for fixed $\theta$ and
$R_A(s,\theta,a)$ is supermodular in $(\theta, a)$ for fixed $s$.
\end{enumerate}
Now, we state the existence of a monotone best response of the agent in the following theorem.
\begin{theorem} \label{Theorem agent}
   In a Bayesian persuasion model with an MLR signaling strategy, there exists a sufficient set of conditions (B1), (B2), (A3), (A4)  by which  the following hold simultaneously:

\begin{enumerate}
    \item The agent's action value function, $Q_A$, is supermodular:
    \begin{equation*}
        Q_A(s,g,a) \text{ has increasing differences in } (s,a) \text{ and } (g,a)
    \end{equation*}
    \item There exists a monotone agent policy:
    \begin{align*}
        a^*: S \times G \to A \ \textnormal{such that} \ a^*(s',g') \geq a^*(s,g) \\ \forall   s' \geq s, g' \geq g
    \end{align*}
    \item The agent's best response is monotone, and the signaling strategy is incentive compatible.\\
    For every $s,g$ and every deviation from monotone action, $a' \in \mathcal{A}$,
\[
Q_A(s,g,a^*(s,g)) \ge Q_A(s,g,a').
\]
\end{enumerate}
\end{theorem}
Subsequently, in the following section, we utilize the structural results for a computationally efficient learning of a monotonic agent policy as well as a signaling strategy in dynamic Bayesian persuasion.

\section{Structured Learning}\label{Structured  Agent's learning}
In this section, we present Algorithm 1 (MAPL), which exploits the monotonic structure for faster learning of the agent policy. Then, Algorithm 2, Supermodular Q learning of Principal (SQP) utilizes the supermodular structure of the principal’s Q-value to enable computationally efficient design of the signaling strategy. Furthermore, Algorithm 3 presents an online reinforcement learning framework for   signaling strategy that remains persuasive for a monotonic learning agent.
\subsection{Structured policy learning of the monotone agent under Bayesian persuasion } \label{Threshold learning}
 Structure-based learning leverages the threshold characteristics of the optimal agent policy \cite{marbach2001simulation}, focusing solely on the set of monotonic policies to approximate the optimal solution. It specifically learns the threshold where the optimal action shifts. The threshold is updated based on the gradient of the expected reward received by the agent.   The average reward defined in (\ref{agent_exp_cum_cost}) can be rewritten based on the evaluation of different policies parameterized by $\sigma$,  as follows: 
\begin{align}
\tilde{R}_A=\lim_{t \to \infty} \frac{1}{t} \mathbb{E}_{\sigma}[\sum _{i=0}^{t}\tilde{r}_{A}(s_{i};\sigma)]
\end{align}
Here, $ s_i$ denotes the state at a particular time, $\tilde{r}_{A}(s_{i};\sigma)$ is the expected immediate reward based on the state, posterior belief of external state and agent's actions under $\rho_{\sigma}$ and the expectation is computed with respect to transition probability, signaling strategy and prior belief distribution.
\begin{lemma}\label{grad_lemma}
The gradient of the average reward of the monotone  agent following a policy parameterized by $\sigma$ is given by
\begin{align}\label{grad_update}
\nabla_\sigma  \tilde{R}_A= \sum_{s \in \mathcal{S}}  d(s; \sigma) & \cdot \nabla_\sigma \tilde{r}_A(s) \  + \\ &\sum_{s \in \mathcal{S}}\sum_{s' \in \mathcal{S}} d(s; \sigma) \nabla_\sigma P(s' \mid s; \sigma)\cdot \tilde{V}_A(s')\notag
\end{align}
where $d(s)$ is the stationary state distribution,\\
$\tilde{V}_A (s')= \sum_{g' \in \mathcal{G}}\sum_{\theta' \in \Theta}\mu_{s'}(\theta')\cdot \PI{s'}{\theta'}{g'} \cdot V_A(s',g')$ 
\end{lemma}
The proof of the lemma is given in the Appendix. 
The above expression shows that the gradient of the long-term average reward depends on expected immediate reward, how the transition dynamics shift under $\sigma$ (i.e., $\nabla_\sigma P$), and how the states being transitioned are (i.e., $\tilde{V}_A(s')$), weighed by the stationary distribution $d(s)$. This is used in the following section for policy gradient methods for learning agent's policy under Bayesian persuasion.
\subsection{Online learning of an agent's monotonic policy under Bayesian persuasion }
In this section, we introduce an algorithm that updates the parameterized agent policy at each time step.
This approach enables us to decompose the process into a series of incremental updates performed at every time step \cite{marbach2001simulation}.  In the context of Bayesian persuasion, we consider the expected reward of the agent's action based on the transmitted signal, under the principal's committed signaling strategy. Here, we substitute the gradient with a biased estimate derived from simulating a single sample path.  It is known that this bias diminishes asymptotically, resulting in convergence \cite{marbach2001simulation}.  We impose the following assumption regarding the transition probabilities, i.e. for every parameter $\sigma$ and state $i,j$, there exists a bounded function $L_{ij}$ such that, $\nabla P_{ij}(\sigma)=  P_{ij}(\sigma) \cdot L_{ij}(\sigma)$
The gradient in (\ref{grad_update}) can be written  as :
\begin{align*}
 \nabla \tilde{R}_A(\sigma)
= \sum_{n=0} ^{T-1}\nabla \tilde{r}_{A}({s_n};\sigma)+ \sum_{n=1} ^{T-1}(\tilde{V}_{A_{s_n}}(\sigma, \tilde{R}_A)L_{s_{n-1}s_n}(\sigma)) 
\end{align*}
\begin{align*}
\nabla \tilde{R}_A(\sigma)=& \nabla \tilde{r}_A({s_0};\sigma) + \sum_{k=1} ^{T-1} \nabla \tilde{r}_A({s_k};\sigma) +(\tilde{r}_A({s_k};\sigma)- \tilde{R}_A)z_k \notag 
\end{align*}
\begin{align*}
\text{ where}, z_k=\sum_{n=1}^{k} L_{s_{n-1}s_n}(\sigma) \notag  
    =\sum_{n=1}^{k} \frac{\nabla P_{s_{n-1}s_n}(\sigma)}{P_{s_{n-1}s_{n}(\sigma)}}
\end{align*}
The parameters can be updated at each step, $t$, recursively as,
\begin{align}\label{recur}
 z_0 =& 0  \\
   z_{t+1}=& z_t + L_{s_t s_{t+1}}(\sigma) \notag \\ 
    =& z_t + \frac{\nabla P_{s_t s_{t+1}}(\sigma)}{P_{s_t s_{t+1}}} \   \notag \\  
  \sigma_{t+1}= &\sigma_{t} +  (\nabla \tilde{r}_{A}(s_t;\sigma_t)+(\tilde{r}_{A}(s_t;\sigma_t)- \tilde{R}_{A,t})z_t )\notag \\ 
   \tilde{R}_{A,t+1}= &\tilde{R}_{A,t} + \eta (\tilde{r}_A(s_t;\sigma_t) - \tilde{R}_{A,t})\notag 
\end{align}
  When the horizon is substantial, $z_t$ will also grow large before being reset to zero, leading to increased variance in the updates. This is a common challenge encountered with likelihood ratio methods. Therefore, it may be advantageous to incorporate a forgetting factor\cite{marbach2001simulation}. In order to reduce the variance, we introduce a discount factor of $\beta$ for $z_t$, and it becomes, 
  $ z_{t+1}= \beta* z_t + L_{s_t s_{t+1}}(\sigma) \  $
\subsection{Parameterization of agent's monotonic policy}
In this section, we consider the approximation of the optimal agent policy, which is monotonically increasing in the state and signal spaces, using parametric curves.
When there are two actions, the optimal monotonic policy in continuous state and signal spaces can be approximated by a threshold curve characterized by a set of parameters, $\sigma,$ which partitions the region spanned by the state and signal space into two. The optimal policy, which increases with state and signal space, can be characterized by a monotonic threshold curve. The parametric representation of monotonic policy can utilize step, piecewise linear, or sigmoidal approximations. The continuity and differentiability of the sigmoid at all points make it an ideal choice to represent a monotonic curve.
 The sigmoidal approximation of an optimal monotonic policy in state and signal spaces can be given by 
\begin{equation} \label{thres_sig}
 \tilde{s}(\sigma,g) =\frac{\phi_2-\phi_1}{1+e^{-\phi_3+g\cdot \phi_4 }}+\phi_1    
\end{equation} 
where each sigmoid is characterized by a set of parameters given by $\sigma=\{ \phi_1,\phi_2,\phi_3,\phi_4\}$ \cite{saureduced2015simulation}. $\phi_1$, $\phi_2$ represent the floor and ceiling of the sigmoid, respectively. The choice of $\phi_4 $ determines the direction and spread of the sigmoid. The curve is decreasing in the state-signal space if $\phi_4>0$, constant if $\phi_4=0$, and increasing if $\phi_4<0$. The change in curvature of the sigmoid occurs at $[\frac{\phi_1+\phi_2}{2},\frac{\phi_3}{\phi_4}]$. 
The optimal  threshold policy approximation is the solution of \begin{equation}
    \label{cum_sig}
    \sigma^*= \arg \max_\sigma \tilde{R}_A(\rho(\sigma))\end{equation}
where $\tilde{R}_A$ is the expected average reward given in (\ref{agent_exp_cum_cost}).
\begin{algorithm}[htbp]
\caption{Monotonic Agent Policy Learning (MAPL)
}\label{alg:Struct_agent}
\hspace{2em}
\begin{algorithmic}[1]
\Require  Horizon length $N$

    \State  Signaling strategy $\gets \boldsymbol{\pi}$ 
    \For{steps $t=1,2,3....N$} 
        \State Observe the external parameter $\theta_t \gets \MU{s_t}{.}$
        \State Obtain the signal $g_t \gets \PI{s_t} {\theta_t}{.}$
        \State Obtain the action $a_t$ based on $\rho_\sigma(s_t,g_t)$
        \State Compute expected immediate reward $\tilde{r}_{A}(s_t;\sigma_t)$ 
        \State Update the agent's threshold policy parameter
        \begin{align*}
 &z_{t+1}= 
     z_t + \frac{\nabla P_{s_t s_{t+1}}(\sigma)}{P_{s_t s_{t+1}}} \   \notag \\  
  &\sigma_{t+1}= \sigma_{t} +  (\nabla \tilde{r}_{A}(s_t;\sigma_t)+(\tilde{r}_{A}(s_t,\sigma_t)- \tilde{R}_{A,t})z_t) \notag \\ 
   &\tilde{R}_{A,t+1}= \tilde{R}_{A,t} + \eta (\tilde{r}_A(s_t;\sigma_t) - \tilde{R}_{A,t})\notag 
\end{align*}
        
        \State Get the next state $ s'$ given by $P(s_t,.,a_t)$
        \State $s_{t+1} \gets s'$
    \EndFor
\end{algorithmic}
\end{algorithm}

Algorithm \ref{alg:Struct_agent} (MAPL) describes the structured online learning of the agent policy. Given an initial state, depending on the prior belief, an external state is sampled. Based on the signaling strategy, a signal is then sampled for the realised state and external state. The principal transmits the signal and this updates the agent's posterior belief of external state. Based on the expected immediate reward, the agent updates the threshold policy and chooses an action. This process is repeated until the agent's policy converges for the committed signaling strategy of the principal.
In the next subsection, we investigate how the Q structure of the principal can be leveraged for the design of a computationally efficient signaling strategy for a monotonic agent.
\subsection{Supermodular Q learning of principal (SQP)}
Here we discuss how standard relative Q learning \cite{yang2024relative} is modified to handle constraints so that it can adaptively learn a signaling strategy persuasive for a monotonic agent.
We can approximate the Bellman equation for the principal given by (\ref{Principal Q}) for relative Q learning as :
\begin{align*}
 & Q_P(s,\theta,g)= Q_P(s,\theta,g) +  \alpha' [ r_P(s,\theta, g) + \\ 
 & \max_{g'} Q_P(s',\theta', g') -  Q_P(s,\theta, g)   - \max_{g'} Q_P(s_0,\theta_0, g')]
\end{align*}
where $\alpha'$ is the learning rate and $(s_0,\theta_0)$ is the reference state.  The optimal solution can be computed as 
\begin{align}
\label{opt_Q}
\mathbb{E}  \left[ r_P(s,\theta, g) + \max_{g'} Q_P^*(s',\theta', g') )-  Q_P^*(s,\theta, g) \right]=0 
\end{align}
According to Theorem 3, principal's Q-values exhibit supermodularity. 
Imposing this supermodular structure enables efficient  learning while  
bypassing non-optimal strategies \cite{krishnamurthy2016partially}.
The convergence of supermodular Q-learning is well-understood \cite{djonin2007q} and extends naturally to our context.
The supermodularity as a linear inequality constraint on $Q_P $ is :
\begin{small}
\begin{align*}
Q_P(s,\theta,g+1)-Q_P(s,\theta,g) \leq  Q_P(s', \theta, g+1) - Q_P(s', \theta, g) \\
Q_p(s,\theta,g+1)-Q_P(s,\theta,g) \leq  Q_P(s, \theta', g+1) - Q_P(s, \theta', g) \notag
\end{align*}
\end{small}
$\forall s \leq s', \theta \leq \theta'.$ 
The constraint on the supermodularity of Q factors can be given as:
\begin{align}\label{eq_ineq}
 Q_PI \leq 0 
\end{align}
where $ \leq$ is elementwise.
 Let $f(Q_P(s,\theta,g))$ be taken as
 \begin{small}
\begin{align*} 
f(Q_P(s,\theta,g))=   (r_P(s,\theta, g) +  \max_{g'} Q_P(s',\theta', g') )- Q_P(s,\theta, g)  \notag 
\end{align*} 
\end{small}
Also let  $ f(Q_P)=\nabla_{Q_P}h(Q_P)$\\
Then,  by (\ref{opt_Q}), we  have $\mathbb{E}({f(Q_P^*)})=\nabla_Q\mathbb{E}({h(Q_P^*)})=0$.\\ Hence, we consider Q learning as stochastic gradient algorithm \cite{krishnamurthy2016partially} that maximizes an objective 
$h(Q)$ subject to (\ref{eq_ineq}), i.e. 
\begin{align}
Q_P^*= \arg \max_{Q_P}\mathbb{E}(h(Q_P))
\end{align}
Now, we use primal-dual stochastic algorithm \cite{krishnamurthy2016partially} to approximate $Q_P^*$ with $\lambda \geq 0$ as  Lagrange multipliers given by:
\begin{align}\label{con_Q}
Q_{P, t+1}=&Q_{P,t}+  \alpha' [f(Q_{P,t})  - \max_{g'} Q_{P,t}(s_0,\theta_0, g') -  \lambda_t I ] \\
\lambda_{t+1} =& max[\lambda_t + \alpha' Q_{P,t} I, 0] \notag
\end{align}
\begin{algorithm}[htbp]
\caption{Supermodular Q learning of Principal (SQP)}
\label{alg:super_q_learning}
\begin{algorithmic}[1]
\Require Learning agent's best response  policy $\{\tilde\rho_k\}$, Principal's Q value  $Q_{P,k}$
\State Initial state $s_0$
\For{$t = 0, 1, 2, \dots T$}
    \State $\theta_{t} \gets \mu(\cdot \mid s_t)$
    \State $\pi(. | s,\theta) = \frac{\exp\left( Q_P(s,\theta,g) / \tau \right)}{\sum_{g'} \exp\left( Q_P(s,\theta, g') / \tau \right)}$
    \State $g_t \gets \pi(\cdot \mid s_t,\theta_t)$
        \State Choose action $a_t \sim \tilde{\rho_k} (\cdot \mid s_t,g_t)$
        \State Observe next state $s_{t+1}\gets P(\cdot \mid s_t,a_t)$
        \State Observe $R_{P}(s_t, \theta_t, a_t), R_{A}(s_t, \theta_t, a_t)$
        \State Update $Q_{P,t+1}$ and ensure supermodularity   using (\ref{con_Q})
        \EndFor
\end{algorithmic}
\end{algorithm}
 Algorithm 2 (SQP) describes the online learning of the principal's Q value for the agent's best response policy $\tilde\rho$. At each time step, for the given $(s_t,\theta_t)$, a signal $g_t$ is sampled by taking the soft max of the principal's Q values. On receiving the signal, the agent takes an action based on $\tilde{\rho}$ and the rewards of both the principal and agent are observed. The Q values of the principal are then updated using (\ref{con_Q}). This is continued till the end of the episode.
 \subsection{Repeated principal agent interaction}
 \begin{figure}[htbp]
    \centering
   \includegraphics[width=8.5cm,height=3.8cm]{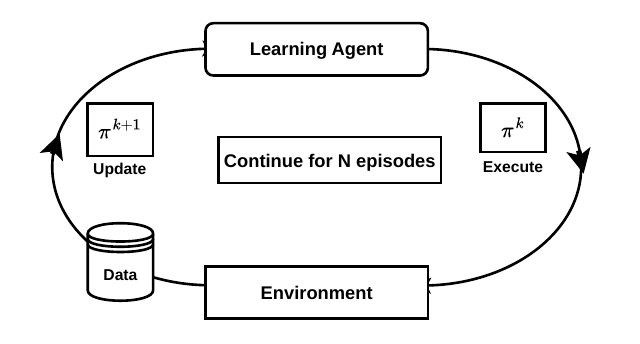}
    \caption{The figure illustrates an online learning framework in which the principal and the agent iteratively adapt their strategies through repeated interactions. At each episode, given the signaling strategy announced by the principal, the agent learns its optimal policy. The principal then observes the agent’s best response and updates the signaling strategy accordingly for the subsequent episode.}
    \label{p_a}
\end{figure}
 In this paper,  we consider the dynamic Bayesian persuasion as illustrated in Fig 3, wherein
the principal repeatedly interacts with an approximately best responding learning agent.  
Let $\tilde{R}_P(\pi,\rho)$ and $\tilde{R}_A(\pi,\rho)$ denote the principal’s and agent’s
long-run expected average rewards, respectively, under the signaling strategy $\pi$ and agent's policy $\rho$. For any fixed signaling strategy $\pi^k$, the agent selects a best response at $k^{th}$ episode by choosing
\[
\rho^k =\arg\max_{\tilde\rho} \tilde{R}_A(\pi^k,\tilde\rho)
\]
Given the induced agent response, the principal updates the signaling strategy
accordingly as
\[
\pi^{k+1} =\arg\max_{\tilde\pi} \tilde{R}_P(\tilde\pi,\rho^k)
\]
Initially, given the principal's signaling strategy $\pi^k$, the agent learn's over an episode and finds its best response, $\rho^k$. Observing the agent's best response, the principal modifies its signaling strategy for the next round, $\pi^{k+1}$, and the repeated iteration continues.
$\pi^*$ is said to be dynamically stable if
$\nabla_\pi \tilde{R}_P(\pi,\rho)=0$ as $\pi^k \to \pi^*$.
\begin{algorithm}[h]
\caption{Online learning of signaling strategy and agent's policy}
\label{alg:joint_learning}
\begin{algorithmic}[1]
\State \textbf{Initialize:} Signaling strategy $\pi_0$, agent policy  $\tilde{\rho_0}$, step sizes $\{\alpha'\}$
\For  {$k = 0,1,2,\ldots N$}
  \State 
    Principal  commits to signaling strategy $\pi_{k}$
\For {$t = 0,1,2,\ldots T$}
    \State Agent policy update using MAPL
    \[
     \sigma_{t+1}= \sigma_{t} +  (\nabla \tilde{r}_{s_t}(\sigma_t)+(\tilde{r}_{s_t}(\sigma_t)- \tilde{R}_{A_t}))z_t \notag 
    \]
\EndFor
\For {$t' = 0,1,2,\ldots T$}
  \State Update  principal's Q value under agent's best \ \quad response using SQP
    \begin{align*}
Q_{P, t'+1}=Q_{P,t'}+  \alpha' [f(Q_{P,t'})- \lambda_{t'} I ] 
\end{align*}
\EndFor
\State Output signaling strategy $\pi$ using (\ref{bolt_signaling})
\EndFor
\end{algorithmic}
\end{algorithm}
\subsection{Structured learning for repeated principal agent interactions}
Algorithm 3  describes the online learning of the signaling strategy and agent's policy. At the start of each round, the sender publicly commits to an incentive-compatible signaling strategy 
$\pi$ for a monotonic agent.
Knowing the signaling strategy, agent learn to select an optimal policy that maximizes expected reward using MAPL which updates the threshold at each step based on the expected reward. After observing the agent’s action for an episode, the principal computes the  average reward and updates its Q value under agent's best response using SQP. Thus, using structured learning, Algorithm 3 learns the agent’s monotonic policy and leverage the supermodular structure of the value function to ensure that the sender’s signaling strategy remains persuasive throughout the repeated interaction.
\section{Numerical Results: Bayesian Persuasive Driving}\label{Numerical results}
\begin{figure}[htbp]
    \centering
    \includegraphics[width=0.5\textwidth]{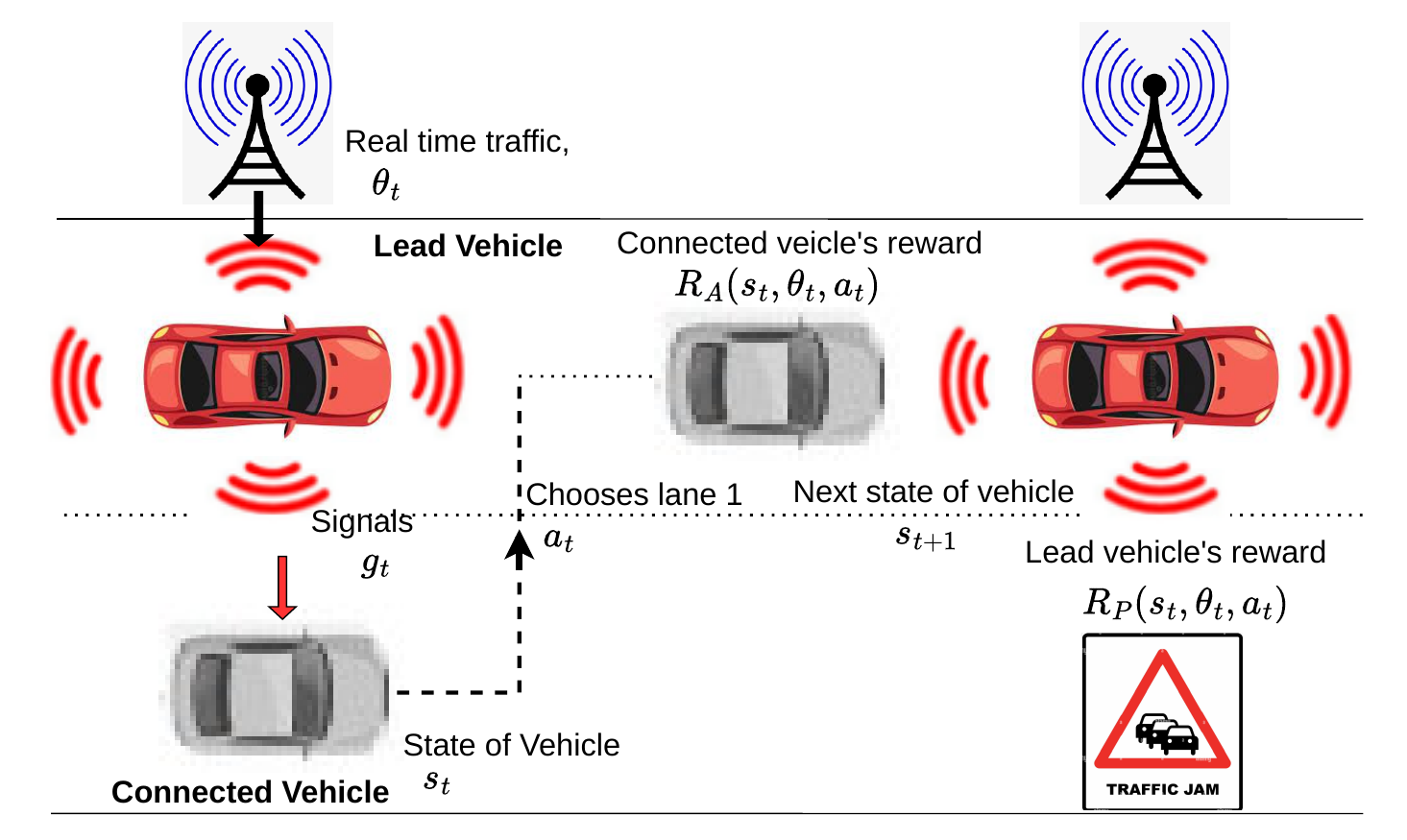}
    \caption{The figure gives an illustration of Bayesian persuasive driving. Here, the red coloured vehicle symbolises the lead vehicle and blue represents the connected vehicle. 
    Road side Units(RSU) transmit traffic data to the lead vehicle.
    The lead vehicle who is aware of the real-time traffic updates guides the connected vehicle through the dark line ensuring congestion-free traffic. 
    }
    \label{traff_rout}
\end{figure}
In this work, we consider lane-based routing—where decisions are made not just at the road level but at the granularity of individual lanes— which is becoming increasingly viable due to advancements in vehicle-to-vehicle (V2V) communication, real-time traffic sensing and precise localization technologies.
\subsection{Real-time traffic routing via Bayesian persuasion}
\label{Real-time traffic routing using Bayesian persuasion in continuous spaces}
We model a dynamic Bayesian persuasion framework for intelligent, lane-based traffic routing involving an intelligent lead vehicle (the principal) and connected vehicle (agent), as depicted in Fig.~\ref{traff_rout}. 
The framework includes:
\begin{enumerate}
    \item \textbf{Intelligent lead vehicle}:  
    The lead vehicle has access to real-time traffic conditions—including accidents and road closures—through GPS, RSUs, and other external sources~\cite{das2017reducing}. Based on this information, it strategically selects signaling strategies  to guide vehicles into specific lanes, aiming to optimize overall traffic flow.
    \item \textbf{Autonomous Connected Vehicle}:  
    Vehicle use onboard sensors to monitor their speed and location, while receiving external updates from the lead vehicle. Based on the observed signal and knowledge of the signaling strategy, it updates the belief about  congestion ahead and makes lane-selection  locally and in real-time.
    \item \textbf{Continuous feedback}: The routing process is adaptive and continuous, with connected vehicles being updated regularly based on real-time traffic conditions. The dynamic change in the state of the connected vehicle is also continuously monitored by the lead vehicle to provide further updates, improving overall efficiency and service quality.
\end{enumerate}
We consider the following model for Bayesian persuasion ~\cite{das2017reducing, peng2019bayesian}.
\begin{itemize}
    \item \emph{State Space}:  
    The vehicle state represents the speed levels of the connected vehicle. Here $s \in \mathcal{S} = \{0,1,2 .,10\}$

    \item \emph{External State Space}: Under typical traffic conditions, vehicles move at free-flow speed, which is the speed attainable when traffic density is low and movement is not affected by surrounding vehicles. However, if many vehicles choose the same route at the same time, their collective demand can surpass the lane’s capacity, causing congestion. In our model, the external state represents the extent of this congestion.
    The external state quantifies traffic congestion based on the traffic  flow model in~\cite{pan2022model}, where:
    \begin{align}
    \theta = \frac{v_d}{v_f} = \left(1 - \left(\frac{d}{d_{\text{jam}}}\right)^p \right)^q
    \end{align}
    Here, \( d \) is traffic density, \( d_{\text{jam}} \) the jam density, \( v_d \) the effective speed at density \( d \), and \( v_f \) the free-flow speed. Here,   increasing levels of congestion-free traffic is represented as $\theta \in \Theta = \{0,1,2..10\}$

    \item \emph{Action Space}:  
   In our traffic routing framework, we analyze a single-link setup that may contain several parallel lanes, each characterized by distinct speed profiles or congestion levels. The available actions correspond to choosing among these lanes. Specifically, in the two-lane case, we define action 1 as selecting the low-speed lane and action 2 as selecting the high-speed lane.
   \item \emph{Agent's travel reward function}: The reward function for choosing lane 1 and lane~2 are  assumed as $R_{A_1}(s,\theta,1)$ as $175 +  s^2 + \theta^2 $ and $R_{A_2}(s,\theta,2)$ as $155 +1.5\cdot s^2 + 1.5 \cdot \theta^2$. The reward structure  ensures that the first lane is better for lower speed and lower traffic-free flow. In addition, the second lane becomes preferable at higher speeds and congestion-free traffic due to its greater rate of reward in comparison with the first.   
  \item \emph{Principal's reward function}: The reward function for choosing lane 1 and lane~2 are  assumed as $R_{P_1}(s,\theta,1)$ to be $R_{A_1}(s,\theta,1)$ and $R_{P_2}(s,\theta,2)$ as $ 500 + R_{A_2}(s,\theta,1)$. This accounts for the extra reward the lead vehicle receives when the connected vehicle chooses the high-speed lane. It satisfies (A1) and (A2).
  \item \emph{Vehicle dynamics}: The transition probability density reflects the likelihood of transitioning to different speeds based on the chosen lane. Low and high-speed ranges are referred to as speed which is less than 50\% of the maximum level possible and the rest respectively. High-speed lane offers higher chances of transitioning to higher speeds than low-speed lanes. The transition probability density captures how likely a vehicle is to shift between different speeds depending on the selected lane, as illustrated in Fig. \ref{fig:transition_diagram}. It reflects the assumption that vehicles traveling at higher speeds tend to maintain those speeds, and similarly for lower speeds. Compared to low-speed lanes, high-speed lanes provide a greater probability of transitioning into and staying within higher speed ranges. The speed transition dynamics of the connected vehicle are consistent with (A4) and (A5).
 \item \emph{Connected vehicle's objective}:
In the proposed framework, the objective of the principal (lead vehicle) is to monitor the speed levels of connected vehicles across different lanes and provide routing signals based on real-time congestion ~\cite{he2015congestion}. Its signaling strategy is designed to guide vehicles predominantly toward high-speed lanes, thereby enhancing overall traffic efficiency while also maximizing the travel reward for the connected vehicle. The agent (connected vehicle) on the other hand aims to determine a policy that selects the optimal lane to maximize its expected average reward favouring higher speeds and minimal congestion.
 \end{itemize}
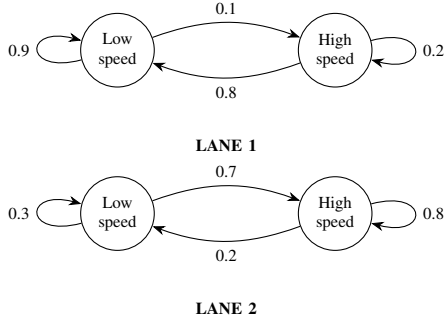
\begin{figure}[!t]
    \centering
    \begin{tikzpicture}[
        scale=0.9, %
        every node/.style={transform shape},
        state/.style={
            circle, draw, minimum size=0.9cm, %
            font=\scriptsize, align=center
        },
        >=Stealth,
        every node/.append style={font=\scriptsize}
    ]

    \begin{scope}[xshift=-0.7cm]

    \node[state] (L1low) at (0,0) {Low\\speed};
    \node[state] (L1high) at (3.2,0) {High\\speed};
    \node[font=\scriptsize, yshift=-0.6cm] at (1.6,-0.8) {\textbf{LANE 1}};

    \path[->] (L1low) edge [loop left, looseness=8] node[left]{0.9} (L1low);
    \path[->] (L1high) edge [loop right, looseness=8] node[right]{0.2} (L1high);

    \path[->] (L1low) edge[bend left=20] node[above]{0.1} (L1high);
    \path[->] (L1high) edge[bend left=20] node[below]{0.8} (L1low);

    \node[state] (L2low) at (0,-2.4) {Low\\speed};
    \node[state] (L2high) at (3.2,-2.4) {High\\speed};
    \node[font=\scriptsize, yshift=-0.6cm] at (1.6,-3.2) {\textbf{LANE 2}};

    \path[->] (L2low) edge [loop left, looseness=8] node[left]{0.3} (L2low);
    \path[->] (L2high) edge [loop right, looseness=8] node[right]{0.8} (L2high);

    \path[->] (L2low) edge[bend left=20] node[above]{0.7} (L2high);
    \path[->] (L2high) edge[bend left=20] node[below]{0.2} (L2low);

    \end{scope}
    \end{tikzpicture}

    \caption{State transition diagram for the two-lane system.}
    \label{fig:transition_diagram}
\end{figure}
We consider a Bayesian persuasion model with a completely unknown reward and transition dynamics, which fit the model assumed for monotonic agent policies, and consider the agent's prior belief as uniform, which satisfies (A3). The assumed Bayesian persuasive driving context satisfies assumptions for a monotonic agent policy and the principal's supermodular Q function. Hence, we utilize the structural results and demonstrate its advantage in  the efficient online learning of the connected vehicle policy as well as signaling strategy of the lead vehicle with the existing methods in the following.
\subsection{Comparison with existing dynamic vehicle routing algorithms}
Here, we compare the travelling reward of the connected vehicle, which uses the Bayesian persuasion framework for lane selection, employing the MAPL with the existing vehicle routing methodologies. The dynamic vehicle routing algorithms considered are:
\begin{enumerate}
\item Model Predictive Control (MPC): 
Here we use the transition model described in  Fig. \ref{fig:transition_diagram}.
The model is iteratively learned during the learning to simulate the future trajectory over a horizon of 5 \cite{ramezani2012estimation}, 
\cite{prabu2022novel}, \cite{hewing2020learning}. At each decision point, MPC computes the cumulative travel reward over the prediction horizon and chooses the action that minimizes it \cite{zhang2018lane}, \cite{pan2022model}. This is continued till the end of the episode.
\item Stochastic Shortest Path (SSP): This approach utilizes the prior assumption of the agent regarding the congestion of the lane ahead. The transition probabilities capture stochastic speed variations, and then one-step prediction is employed to choose a lane which maximizes the reward of travel \cite{levering2022framework}.
\item Tree-based learning: Tree-based learning approximates the value function in dynamic vehicle routing by using trees to estimate the future reward of different routing decisions based on current states. The tree is trained on samples collected through simulation\cite{ulmer2018budgeting}.
\end{enumerate}
\vspace{-.5em}
\begin{center}
\begin{table}[h]
\centering
\caption{Comparison of the travelling reward of connected vehicle for various vehicle routing algorithms}
\label{Table3}
\scalebox{0.68}{\begin{tabular}
{|p{1.1cm}|p{4.8cm}|p{3cm}|}
\hline
Number& Vehicle routing Algorithm & Travelling reward \\[0.2 cm]
 \hline
3&Model Predictive Control&230\\[0.2 cm]
\hline
3&Stochastic Shortest Path&220\\[0.2 cm]
\hline
4&Tree based learning&230\\[0.2 cm]
\hline
5&Q learning&240\\[0.2 cm]
\hline
6& \textbf{MAPL}&\textbf{257}\\[0.2 cm]
\hline
\end{tabular}}
\end{table}
\end{center}
\begin{figure}[htbp]\label{fig agent learning}
\centering
\includegraphics[width=8cm,height=5.5cm]{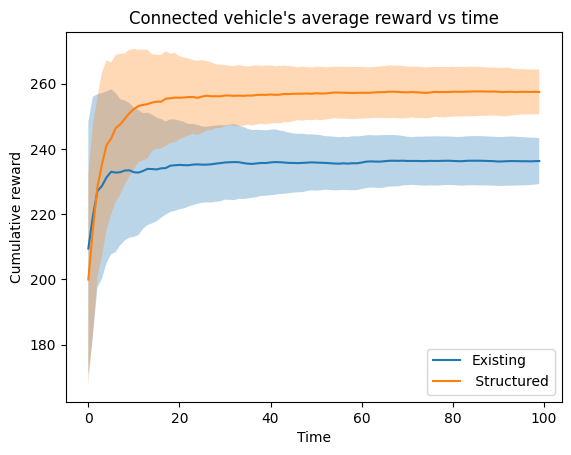}
\caption{Average reward accumulated by the connected vehicle }
\end{figure}
\begin{figure}[htbp]\label{fig principal learning}
\centering
\includegraphics[width=8cm,height=5.5cm]{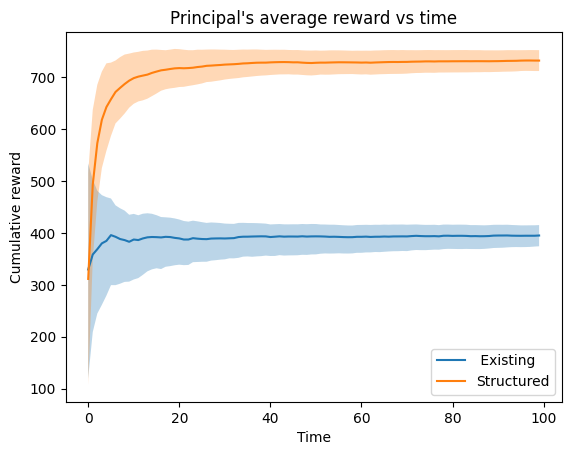}
\caption{Average reward accumulated by the   lead vehicle}
\end{figure}
Here, the proposed structured learning method yields better average rewards considering the farsightedness of the agent, as illustrated in Table 1.  SSP, MPC and Tree based learning obtain lower rewards due to their reliance on prior traffic assumptions without real-time updates. This shows the benefit of utilising the threshold policy of the connected vehicle navigating using the Bayesian persuasion framework for lane selection over the existing dynamic vehicle routing recommendations.
\subsubsection{Comparison with existing methodology for the design of signaling strategy}
In the structured learning, we utilize the monotonic structure of the agent optimal policy and supermodular Q learning of principal for computationally efficient learning of signaling strategy persuasive for connected vehicle. For comparison with existing reinforcement learning framework, the signaling strategy  is evaluated over 100 episodes of 100 steps. We compare with  existing  signaling strategies used in the context of Bayesian persuasion which include
\begin{enumerate}
\item Optimism-Pessimism Principle for Persuasion Process(OP4) \cite{wu2022sequential}: The principal designs signals to optimize  agent's expected immediate  reward, without accounting for future state transitions. Here, pessimism tackles the uncertainty in the prior estimation by selecting a signaling policy that is persuasive for a myopic agent with respect to all the priors in the confidence region, while optimism in principal's Q-function estimation encourages exploration\cite{wu2022sequential}.  
 It perform poorly with far-sighted drivers who anticipate future congestion and signal disclosures.
\item History based online persuasion: 
In traffic routing, the lead vehicle learns a decision tree mapping  histories to signals, anticipating drivers’ future reactions. The IC constraint  requires that at every node in this tree, the agent with full recall — knowing exactly which path led to the current node — still finds it rational to follow the history based signaling recommendation rather than deviate \cite{bernasconi2022sequential}. This approach assumes that model assumptions except the agent's prior belief of the external state are known, and the agent needs to retain the history of past interaction, making it slow to learn, adapt, and plan ahead.
\item Uniform signaling: Here the principal sends signals that are independent of the realized state of nature. As a result, every signal induces the same posterior belief for the agent, equal to the prior. This means the principal does not exploit informational asymmetries and exerts no influence on the agent’s action through belief manipulation. Uniform signaling is often used as a baseline \cite{lingenbrink2018optimal}.
\item Q learning: In Bayesian persuasion, the principal and the agent can use Q-learning independently to maximize their own long-term rewards under repeated interactions
\item Full Revelation: Full-revelation signaling discloses complete traffic state information (e.g., exact congestion levels on all routes)\cite{lingenbrink2018optimal}.
\end{enumerate}
Here, the proposed structured learning yields better average rewards considering the farsightedness of the agent and principal, as illustrated in the Table 2. The faster and enhanced learning of the proposed method,  which takes into account the structure of the agent's policy and principal's value function, compared with the existing online reinforcement learning for signaling strategy, is shown in Fig 6 and 7. 
\begin{table}
\caption{Comparison of the travelling reward of connected vehicle for existing signaling strategies }
\label{Table3}
\scalebox{0.7}{\begin{tabular}
{|p{0.5cm}|p{4.6cm}|p{3cm}|p{3cm}|}
\hline
Item& signaling Strategies & Connected vehicle's average reward & Lead vehicle's average reward \\[0.2 cm]
  \hline
1&OP4&230&420\\[0.2 cm]
\hline
 2&History based online persuasion&218&618\\[0.2 cm]
 \hline
2&Uniform  signaling&251&538\\[0.2 cm]
\hline
3&Q learning &242& 294\\[0.2 cm]
 \hline
4&Fully revealed signaling&247&527\\[0.2 cm]
\hline
5 & \textbf{Structured learning} &\textbf{257}&\textbf{756}\\[0.2 cm]
\hline
\end{tabular}}
\end{table}
\subsection{Computational efficiency as compared to existing methodology}
Fig. \ref{fig:comp_time} illustrates the computational efficiency of the proposed approach with the existing methodology in terms of the computational time. The existing methodology scales with an increase in the dimension of state and external state space, whereas for the proposed methodology, the agent's policy learning remains independent of the increase in dimensionality of state space. By incorporating
prior knowledge of the threshold structure of the agent’s policy and the principal’s Q-values, it improves learning efficiency while reducing computational complexity.
\begin{figure}[htbp]
\centering
\includegraphics[width=0.7\linewidth]{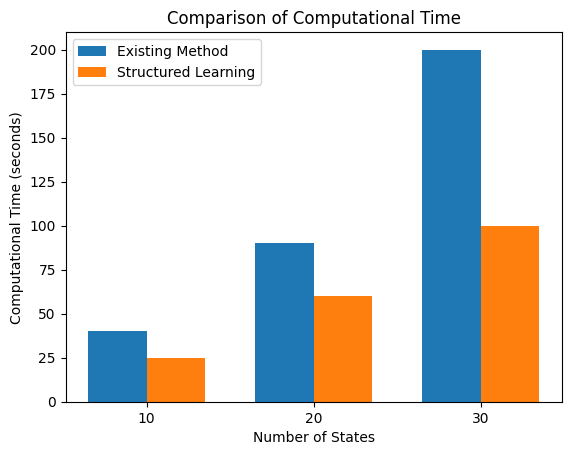}
\caption{Comparison of computational time of existing and structured learning methods}
\label{fig:comp_time}
\end{figure}
\vspace{-2em}
\section{Conclusion}\label{Conclusion} 
Motivated by the efficiency of Bayesian persuasion for intelligent interactive driving, 
this paper proposes an online structured reinforcement learning framework for design of computationally efficient signaling strategy, which is persuasive for a far-sighted agent.
The main contributions are:
\begin{inparaenum} [(i)]
\item MAPL- online structured policy learning algorithm which utilizes the  monotonic structure of  agent  policy for faster learning in  large  state and signal spaces,
\item  Identification of sufficient condition on the Bayesian persuasion model to ensure the supermodular structure of the Q function of the principal for a monotonic agent, 
\item Identification of sufficient conditions for a persuasive signaling strategy for a far-sighted agent
\item Supermodular Q learning of Principal (SQP) for the design of a persuasive signaling strategy  considering a monotonic agent,
\item Case study of Bayesian persuasive driving, which illustrates the practicality of the structural results in addressing  Bayesian persuasion for dynamic vehicle routing.
\end{inparaenum}
By incorporating prior knowledge of the threshold structure of the agent's policy and the principal's Q-values, it improves learning efficiency while reducing computational complexity. This approach is especially effective in highly dynamic real-world applications such as traffic routing, sensor scheduling, inventory management, and transmission scheduling, where the proposed model fits naturally.
The current work considers the homogeneous and monotonic driving behaviour of connected vehicles. However, incorporating the risk-sensitive nature of connected vehicles remains a promising direction for future research.
\vspace{-1em}
\section*{Appendix:}
\subsection{\textbf{Basic definitions}}\label{Definitions}
The following definitions are required for the proof of the main results. 
\begin{definition}(Supermodularity):
The function $r(s,a)$ is said to be supermodular if 
\begin{align}\label{submodular}
r(s,\bar{a})-r(s,{a}) \leq r(\bar{s},\bar{a})-r(\bar{s},{a}) 
\end{align}
where $\bar{s} \geq s$ and $\bar{a} \geq a , \forall s, \bar{s}\in \mathcal{S}, a, \bar{a} \in \mathcal{A}$. In other words, $r(s, \bar{a})-r(s, a)$ is increasing with respect to $s$. 
\end{definition}
 \begin{definition} (Tail sum supermodularity): The probability mass function, pmf,  $P(s,s', a)$, is tail sum supermodular in $(s, a)$ if 
 \begin{align}\label{tail sum}
\sum_{\bar{s}}^S((P(s,s',a+1)- P(s,s',a)) \text{ is increasing in }s  
\end{align}
for $ \ \forall s, \bar{s}\in \mathcal{S},a \in \{0,A\}$
 \end{definition}
\begin{definition} (Stochastic dominance): Let $p_{1}$ and $p_{2}$
be two probability mass functions, then $p_{1}$ is said to first-order stochastically dominate $p_{2}$, denoted as 
$p_{1} \geq_{s}  p_{2}$ if
\begin{align}\label{stoch_dom}
\sum_{\bar{s}}^{S} p_{1}(s) \geq \sum_{\bar{s}}^{S} p_{2}(s)
\end{align}
\end{definition}
\begin{definition} (Monotone Likelihood Ratio (MLR) ordering):
Let $p_{1}$ and $p_{2}$
be two probability mass functions, then $p_{1}$ is said to dominate $p_{2}$ with respect to the MLR order, denoted as  
$p_{1} \geq_{r} p_{2}$ if
\begin{align}\label{MLR}
p_{1}(g_1) \cdot p_{2}(g_2) \leq p_{2}(g_1) \cdot p_{1}(g_2), g_1 \leq g_2, \  \forall g_1,g_2 \in \mathcal{G}
\end{align}
\end{definition}
\begin{theorem}\label{Theorem 2} 
   For any increasing function of $f(.)$. $p_1 \geq_{s} p_2$ iff $\sum_{ s \in \mathcal{S}} f(s) \cdot p_1(s)ds \geq \sum_{ s \in \mathcal{S}} f(s) \cdot p_2(s)ds$ \textnormal{\cite{krishnamurthy2016partially}}.
 \end{theorem}
 \begin{theorem}\label{Theorem 3}
  For any increasing function $f(.)$, the $P(s,s', a)$ is tail sum supermodular iff $\sum_{ s' \in \mathcal{S}} P(s,s', a)\cdot f(s')$ is supermodular in $(s, a)$ \textnormal{\cite{krishnamurthy2016partially}}.
 \end{theorem}
\begin{lemma}\label{lemma_avg}
Consider $\gamma_k$ be an increasing sequence of discount factors $\{\gamma_k\}$ such that $\gamma_k  \rightarrow 1$, and let $\{\rho_{\gamma_k}\}$ be the corresponding  stationary discounted optimal policies. Then there exists a subsequence $\{\zeta_{k}\}$ of $\gamma_k$ and a stationary policy $\rho^*$ which is a limit point of $\{\rho_{\zeta_k}\}$. That is, for every  $(s,g)$, there exists an integer $N(s,g)$ such that for all $n \ge N(s,g), \rho_{\zeta_k}(s,g) =~\rho^*(s,g)$.
\end{lemma}
For the Bayesian persuasion model we assume :
\begin{itemize}
\item [(C1)] Markov chain induced by agent's policy is ergodic. 
\item[(C2)] For every $s \in S, g \in G$ and discount factor $\gamma$ , discounted value function, $V ^{\gamma}_{\rho}(s,g) < \infty$
\item[(C3)] Assuming a reference state $(s_0,g_0)$, there exists a nonnegative $L$ such that \\ $-L \leq h^{\gamma}(s,g)\stackrel{\triangle}{=}  V_A^{\gamma}(s,g)- V_A^\gamma (s_0,g_0) \  \forall (s,g), \gamma$.
 \item[(C4)] There exists a $M(s,g)$ such that  $h^{\gamma}(s,g) \leq M(s,g)$ for every $(s,g)$ and $\gamma$ and an action $a_0$ such that $\sum_{s'\in \mathcal{S}} P(s,s',a_0) \cdot M(s',.)< \infty$.
 \\ Also $\sum_{s' \in \mathcal{S}}P(s,s',a) \cdot M(s',.) < \infty$ for all $(s,g), a$.
\end{itemize}

Under (C1), (C2), (C3), (C4), any policy identified as a limit point in Lemma \ref{lemma_avg}, qualifies as an average reward optimal policy \cite{sennott1989average}. This is established in the following as :
\begin{lemma}
Any stationary deterministic policy $\rho^*$ identified through Lemma \ref{lemma_avg} is average reward optimal. More precisely, there exists a constant $R_A^*= lim_{\gamma \rightarrow 1} (1-\gamma)V_A^{\gamma}(s,g)$ for every $(s,g)$  and $V_A(s,g) $\ with \ $-L \leq V_A(s,g) \leq M(s,g),$ 
such that the Bellman optimality equation
\begin{small}
\begin{align}\label{Bel_cont}
&R_A^* + V_A(s,g) = \max_{a \in A} \biggl[ \tilde{r}_A(s, a) +\\ &\sum_{s' \in \mathcal{S}} P(s' \mid s, a) \sum_{g \in \mathcal{G}} \sum_{\theta \in \Theta}  \MU{s'} { \theta' } \cdot \PI{s'}{\theta'}{g'} \cdot V_A(s',g') \biggl]  \notag
\end{align}
\end{small}
is satisfied where the average reward under the optimal policy $\rho^*$ is  $R_A^*$.
\end{lemma}
\vspace{-1.8em}
\subsection{\textbf{Discussion}}
(C1) ensures a unique stationary distribution and hence the average payoff does not depend on the initial state. (C2) requires that the optimal discounted reward is finite for every state and discount factor. (C3) assumes a reference state (e.g., $(s_0,g_0)$) and a nonnegative function $V_A(s,g)$ uniformly bounding the expected reward to reach it. (C4) states that for each state, there is an action keeping the expected next-stage value within a nonnegative bound $M(s,g)$. It also ensures that one-step transitions are uniformly bounded across all states and actions. These are largely observed in transmission scheduling, traffic routing and queuing models \cite{ngo2009monotonicity}.
\vspace{-1.2em}
\subsection{\textbf{Monotone best responses and incentive compatibility}}
\label{app:Monotone best response}
We need to prove that under sufficient set of conditions, for  MLR signaling strategy $\pi$, the monotone policy $a^*(\cdot,\cdot)$ from Theorem~\ref{Theorem agent} is incentive compatible: for every $s,g$ and every deviation from monotone action, $a' \in \mathcal{A}$,
\[
Q_A(s,g,a^*(s,g)) \ge Q_A(s,g,a').
\]
The proof proceeds with the following steps:
\subsubsection{ Step 1: Posterior belief ordering from MLR}
 \begin{lemma}\label{lemma1}
  If $\pi$ satisfies MLR ordering w.r.t~$g$ then \\ \(\bel{s}{g_i}{\theta} <_r \bel{s}{g_{j}}{\theta}, \  \forall g_i \leq g_j\)
  \end{lemma}
 \begin{proof}
 \textnormal{  MLR ordering of signal distributions implies that posterior belief of external states are ordered in the MLR sense as the signal increases.}
 \begin{small}
\begin{align*}
 \PI{s}{\theta}{g_i}  &\leq_{r} \PI{s}{ \theta} {g_j} \\
\PI{s}{\theta_m}{g_{i}} \cdot \PI{s}{\theta_n}{g_{j}} & \geq \PI{s}{\theta_n}{g_{i}} \cdot \PI{s}{\theta_m}{g_{j}} \\
\textnormal{i.e. \ }\frac{\frac{\PI{s}{\theta_m}{g_{i}}\cdot \MU{s}{\theta_m}}{\sum_{\theta \in \Theta} \PI{s}{\theta}{g_i} \cdot \MU{s} {\theta}}}{\frac{\PI{s}{\theta_n}{g_{i}}\cdot \MU{s}{\theta_n}}{\sum_{\theta \in \Theta} \PI{s}{\theta}{g_i} \cdot \MU{s}{\theta} }}
 &\geq
 \frac{\frac{\PI{s}{\theta_m}{g_{j}}\cdot \MU{s}{\theta_m}}{\sum_{\theta \in \Theta} \PI{s}{\theta}{g_{j}} \cdot \mu_s (\theta) }}{\frac{\PI{s}{\theta_n}{g_{j}}\cdot \MU{s}{\theta_n}}{\sum_{\theta \in \Theta} \PI{s}{\theta}{g_{j}} \cdot \MU{s}{\theta} }} \\
  \frac{b_{s,g_{i}}(\theta_m)}{b_{s,g_{i}}(\theta_n)}
\geq
\frac{b_{s,g_{j}}(\theta_m)}{b_{s,g_{j}}(\theta_n)} , &
 \textnormal{i.e.}, \  \bel{s}{g_i}{\theta} <_r \bel{s}{g_{j}}{\theta} , \  \forall g_i\leq g_j
\end{align*}
\end{small}
\end{proof}
\vspace{-2em}
\subsubsection{Step 2: Ordering of agent's value function}
\vspace{-0.4em}
\begin{lemma}\label{lemma:monotone_sigma}
When 
$\pi$ satisfies MLR ordering with respect to ~$g$ then the normalized measure of Bayesian update,
 $\sigma'(s)$ is first-order stochastically increasing in $s$, i.e.,
\[
\sigma'(s_i) \le_s \sigma'(s_j), \quad \forall s_i \le s_j.
\]
\end{lemma}
\begin{proof}
Fix any threshold $\bar g \in \mathcal{G}$,
\begin{small}
\begin{align*}
\sum_{\bar g}^{G} \sigma'_s(g)\
&=
\sum_{\bar g}^{G} \sum_{\theta \in \Theta} \PI{s}{ \theta} {g} \,\mu_s(\theta)\
\end{align*}
\end{small}
We have  $\PI{s}{\theta}{g_i}  \leq_{r} \PI{s}{ \theta} {g_j}.$
Since $\mu_s$ is first-order stochastically increasing in $s$, it follows that
\begin{small}
\[
\sum_{\bar g}^{G} \sigma'_{s_i}(g)\
\;\le\;
\sum_{\bar g}^{G} \sigma'_{s_j}(g)\,
\quad \forall s_i \le s_j.
\]
\end{small}
Since the inequality holds for all $\bar g$, we conclude that
$\sigma'(s_i) \le_s \sigma'(s_j)$. This result can be used to order the agent's value function as illustrated below:
\vspace{-0.8em}
\begin{lemma} \label{l3}
Let $V_A(s,g)$ be  a increasing function  w.r.t $s$ and~ $g$, then  $\Tilde{V}_A(s_{i}) \leq \Tilde{V}_A(s_{j})$ where 
\begin{align*}
 \tilde{V}_A(s)=\sum_{g \in \mathcal{G}}\sum_{\theta \in \Theta}\MU{s}{\theta} \cdot \PI{s}{\theta}{g} \cdot \tilde{V}_A(s,g) 
 \end{align*}
 $\sigma'(s_i)\leq_{s} \sigma'(s_{j}), \ \forall s_i \leq s_j$
 \\\textnormal{Since $ \tilde{V}_A (s,g) $  is increasing in $g$ then }
 \begin{small}
 \begin{align}
 \label{sigeq_1}
 &\sum_{g \in \mathcal{G}} \!
 \sigma'(s_{i}, g) \!\cdot V_A(s_i ,g)\!  \leq \!\sum_{g \in \mathcal{G}}\! \sigma'(s_{j},\! g)\! \cdot \! V_A(s_{i} ,g) 
\end{align} 
\begin{align}\label{sigeq_2}
\sum_{g \in \mathcal{G}} \sigma'(s_{i}, g) \cdot V_A(s_{j} ,g)  \leq \sum_{g \in \mathcal{G}} \sigma'(s_{j}, g) \cdot  V_A(s_{j} ,g)
 \end{align}
 \end{small}
 \textnormal{Assume $ V_A(s,g)$ is increasing in $s$ and}  $\sigma'(s_i,g) >0 $ 
 \begin{small}
\begin{align}\label{sigeq_3}
 \sum_{g \in \mathcal{G}} \!\sigma'(s_{i} ,g)\! \cdot V_A(s_i ,g)  \leq \sum_{g \in \mathcal{G}}\! \sigma'(s_{i}, g) \! \cdot \! V_A(s_{j} ,g) 
 \end{align}
 \end{small}
 \textnormal{Thus combining (\ref{sigeq_1}), (\ref{sigeq_2}) and (\ref{sigeq_3})} \textnormal{we have}
 \begin{align*}
 \sum_{g \in \mathcal{G}} \sigma'(s_{i} ,g) \cdot  V_A(s_i ,g)  \leq \sum_{g \in \mathcal{G}} \sigma'(s_{j}, g) \cdot V_A(s_{j} ,g)  
 \end{align*}
 \textnormal{i.e.}, \  $\tilde{V}_A(s_i) \leq \tilde{V}_A(s_{j})$
 \end{lemma}
 \subsubsection{Step 3: Supermodularity of agent's Q value function}
Under (B1) and (B2), (A3) and (A4), we have $ Q_A(s,g,a)$ is supermodular in both $(s,a)$ and $(g,a)$. The proof is direct from \cite{paulmonotonic} with the corresponding change while considering supermodularity in reward structures.
\subsubsection{Step 4: Existence of monotonic policy of agent}
According to Theorem \ref{Theorem agent}, since $Q_A(s,g,a)$ is supermodular in $(s,a)$ and $(g,a)$, there exists an optimal policy that is monotonically increasing in the state and signal space.

\subsubsection{Step 5: IC compatibility of MLR signaling strategy}
The agent's best response :
    $a^*(s,g) = \arg\max_{a \in A} Q_A(s,g,a)$
can be chosen to be monotone in both $s$ and $g$.
Thus, given the MLR signaling strategy $\pi$, following the principal's  recommendation is optimal and action deviation cannot yield a higher expected reward. Thus $Q_A(s,g,a^*(s,g)) \ge Q_A(s,g,a')$ for all action deviations from monotone responses $a'$.
This yields \emph{dynamic incentive compatibility}: for every time $t$ and every realized $(s_t,g_t)$,  $a^*(s_t,g_t)$ is the best response. Hence, the MLR signaling scheme with the corresponding monotone policy is incentive compatible for the agent.
\end{proof}

 \subsection{\textbf{Proof of Theorem \ref{Therm main}}}
 Let $V_{P_0}(s,g)$ be the arbitrary initial value function. Then updated $Q_P $ for the $k+1^{th}$ iteration is given as:
\begin{small}
\begin{align}
&Q_{P_{k+1}}(s,\theta,g)=\sum_{a \in \mathcal{A}} \tilde{\rho}(a|s,g) \cdot R_P(s,\theta,a)+  \notag\\  &\sum_{a \in \mathcal{A}} \tilde{\rho}(a|s,g) \cdot \sum_{s' \in \mathcal{S}} P(s,s',a)  \sum_{\theta' \in \Theta}  \mu_{s'}(\theta')  \cdot V_{P_k}(s', \theta')\notag \notag
\end{align}
\end{small}
In order to prove  that $Q_P(s,\theta,g)$ is supermodular in $(s,g)$ for fixed $\theta$ and also supermodular in $(\theta,g)$ for fixed $s$, using mathematical induction we proceed in two steps:
\subsubsection{To prove value function is monotonically increasing in  $s$ and $\theta$}
\begin{proof}
Let us assume that $V_{P_k}(s,\theta)$ is increasing in $s$ and $\theta$. 
By Theorem 
\ref{Theorem 2},  $\sum_{\theta' \in \Theta}  \mu_{s'}(\theta')  \cdot V_{P_k}(s', \theta') $ is increasing in $s'$.
By (A4), $P (s_i,.,a) \leq_s  P(s_j,.,a), \forall i \leq j$.
Hence,
$\sum_{s' \in \mathcal{S}} P(s,s',a)  \sum_{\theta' \in \Theta}  \mu_{s'}(\theta')  V_{P_k}(s', \theta')$ is increasing in s. \\ 
Also $R_P(s,\theta,a)$ is increasing in $s$, $\theta$ and $a$.
By (A6),
$\tilde{\rho}(a|s_i,g) \leq_s \tilde{\rho}(a|s_j,g), \  \forall i<j, $ 
$\sum_{a \in \mathcal{A}} \tilde{\rho}(a|s,g) \cdot R_P(s,\theta,a) $ is increasing with $s$.\\
Similarly $\sum_{a \in \mathcal{A}} \tilde{\rho}(a|s,g) \cdot R_P(s,\theta,a) $ is increasing with $\theta$. Hence their sum, $Q_{P_k}(s,\theta,g)$ is also increasing in $s$ and $\theta$. 
$ V_{P_{k+1}}(s,\theta)
=\sum_{g \in \mathcal{G}} Q_{P_k}(s,\theta,g) \cdot \pi_k(g|s,\theta)$ \\
Assuming $\pi_k(g| s, \theta) $is MLR with respect to $s$ and $\theta$, %
  $V_{P_{ k+1}}(s, \theta)$ is also increasing in $s$ and $\theta.$ \\
Thus, the mathematical induction is complete.The optimal value function is the fixed point solution of the value iteration algorithm, and hence the optimal value function is monotone and increasing in $s$ and $\theta$ is proved.
\end{proof}
 \subsubsection{To prove that $Q_P(s,\theta,g)$ is supermodular in $(s,g)$ and $(\theta,g)$}
\begin{proof}
We have $\mu_{{s_i}}(\theta)  \leq_{r} \mu_{{s_j}}(\theta) , \   \forall i < j$ and $V_{P_k}(s,\theta,g)$ is increasing in $(s,\theta)$. Hence by Theorem ~\ref{Theorem 2}, \\ $\tilde{V}_{P_k}(s')=\sum_{\theta' \in \Theta} \mu_{s'}(\theta')  \cdot V_{P_k}(s', \theta')$ is increasing in $s'$.
$P(s,.,a)$ is tail sum supermodular in $s$ and $a$.
Hence by Theorem \ref{Theorem 3}, $\sum_{s' \in \mathcal{S}} P(s,s',a) \cdot \tilde{V}_{P_k}(s')$ is supermodular in $(s,a)$. \\
Using (A6) and Theorem 5, we have \\ $\sum_a \tilde{\rho}(a|s,g) \cdot \sum_{s' \in \mathcal{S}} P(s,s',a)  \sum_{\theta' \in \Theta}  \mu_{s'}(\theta') \cdot V_{P_k}(s', \theta')$ is supermodular in $s$ and $g$. \\
By (A2), $R_P(s,\theta,a)$ is supermodular in $(s,a)$ and $(\theta,a)$. Again, we have $\sum_a \tilde{\rho}(a|s,g) \cdot R_P(s,\theta,a)$ is supermodular in $s$ and $g$. 
Similarly, we have $\sum_{a \in \mathcal{A}} \tilde{\rho}(a|s,g) \cdot R_P(s,\theta,a)$ is supermodular in $\theta$ and $g$.\\ Since the sum of supermodular function is supermodular we have,
$Q_{P_{k+1}}(s,\theta,g)=\sum_{a \in \mathcal{A}} \tilde{\rho}(a|s,g) \cdot R_P(s,\theta,a)+ \sum_{a \in \mathcal{A}} \tilde{\rho}(a|s,g) \cdot \sum_{s' \in \mathcal{S}} P(s,s',a)  \sum_{\theta' \in \Theta}  \mu_{s'}(\theta')  \cdot V_{P_k}(s', \theta') $ is supermodular in $(s,\theta)$ and $g$. 
Hence, by mathematical induction, the Q  function of the principal is supermodular in $(s,g)$ and $(\theta,g)$.
\end{proof}
\vspace{-1.7em}
\subsection{\textbf{Discussion of assumptions for supermodular Q function}}
\label{Assumptions}
These assumptions are widely observed in transmission control, sensor scheduling, and traffic routing applications, such as prioritizing transmission in good channels, activating sensors when uncertainty is high \cite{ngo2009monotonicity}, \cite{paulmonotonic}.
 To illustrate when these assumptions are applicable, we present a representative example of lane-based Bayesian persuasive driving detailed in Section \ref{Numerical results}. This involves a principal, i.e., lead vehicle, and an agent, i.e., vehicle connected for selection of low speed and high speed lane. The state refers to the speed of the vehicle, and the external state indicates the extent of congestion-free traffic. 
 
 \begin{itemize}
 \item (A1) ensures that there is an increase in reward of the principal at higher speeds, higher congestion-free traffic, and selection of high-speed lanes. It formalizes the intuition that travel rewards increase with higher vehicle speed (state),  congestion free traffic (external state), and lane selection (action).  Related works \cite{nagarajan2016approximation} and \cite{loui1983optimal} also exploit monotonic reward structures to rank paths in dynamic traffic routing. 
 \item (A2) assures the supermodular reward for the principal, which indicates the increased reward for higher speeds and higher congestion-free traffic at higher actions. It ensures the marginal benefit of taking
higher action at higher states and external states. It plays a key role in numerous applications, including traffic routing and resource allocation problems such as joint replenishment and inventory routing\cite{nagarajan2016approximation}, determining optimal locations for mobile recharging stations in unmanned vehicle networks\cite{shi2023decision}, and tasks like coverage maximization, adaptive ranking, and routing\cite{navidi2020adaptive}. 
\item (A3) indicates the monotonic correlation between state and external state, which captures the correlation between speed and congestion-free traffic.
 \item The higher probability of transitioning to more favourable states from higher states and actions are captured in (A4) and (A5). This tendency to remain at higher speeds as current speed increases— under stochastic dominance over states—is commonly modelled in predictive control frameworks\cite{guo2024efficient},\cite{liu2019vehicle}. Moreover transition model in \cite{prabu2022novel} similarly reflects the higher likelihood of reaching greater speeds associated with different speed lanes.
 \item (A6) ensures that the agent's policy is monotonically increasing with respect to state and signal spaces \cite{paulmonotonic}.
 In other words, higher states and signal more effectively favour higher action, which also helps in faster convergence while considering a monotonic agent.
 \end{itemize}
 \vspace{-1.2em}
\subsection{\textbf{Proof of Lemma 1: Softmax policy is MLR under supermodularity}}
\label{lem:softmax-mlr}
 \begin{proof}
 We have $\pi(g\mid s,\theta)=\frac{\exp\big(Q_P(s,\theta,g)/\tau\big)}{Z(s,\theta)},$
where
$ Z(s,\theta)$ is $\sum_{g' \in \mathcal{G}} \exp\big(Q_P(s,\theta,g')/\tau\big).$
By supermodularity, \\
\begin{small}
$
Q_P(s,\bar g)-Q_P(s,g)\leq Q_P(\bar s,\bar g)-Q_P(\bar s,g),
$
\end{small}
and similarly with $\theta$ replacing $s$. 
We need to prove that $\pi$ is MLR in $(s,g)$ and in $(\theta,g)$.  , i.e. for  any $\bar s\ge s$ and any $g_1\le g_2$, \\
\begin{small}
$\pi(g_1\mid\bar s,\theta)\,\pi(g_2\mid s,\theta)
\le
\pi(g_1\mid s,\theta)\,\pi(g_2\mid\bar s,\theta)$,
\end{small}
and likewise with $\bar\theta\ge\theta$. \\
Fix $\theta$ and let $\bar s\ge s$. For any $g_1\le g_2$, write the two products using the softmax form:
\begin{small}
\[
\pi(g_1\mid\bar s,\theta)\pi(g_2\mid s,\theta)
=\frac{e^{(Q(\bar s,\theta,g_1)+Q(s,\theta,g_2))/\tau}}{Z(\bar s,\theta)Z(s,\theta)},
\]
\[
\pi(g_1\mid s,\theta)\pi(g_2\mid\bar s,\theta)
=\frac{e^{(Q(s,\theta,g_1)+Q(\bar s,\theta,g_2))/\tau}}{Z(s,\theta)Z(\bar s,\theta)}.
\]
\end{small}
The normalizers cancel, so the desired inequality is equivalent to
$
Q(s,\theta,g_2)-Q(s,\theta,g_1)\le Q(\bar s,\theta,g_2)-Q(\bar s,\theta,g_1),$
which is exactly the assumed supermodularity with $\bar g=g_2$ and $g=g_1$. Similarly, with $s$ fixed and $\bar\theta\ge\theta$ proves  MLR  in $(\theta,g)$.
 \end{proof}
 \vspace{-2em}
\subsection{\textbf{Proof of Lemma 2}}
\begin{proof}
The stationary distribution satisfies the balance equation:
$d(s';\sigma) = \sum_{s \in \mathcal{S}} d(s;\sigma) P(s' \mid s; \sigma).$
Taking the gradient with respect to $\sigma$, and then multiplying both sides with $\tilde{V}_A(s')$ and taking the sum we get
\begin{small}
\begin{align*}
\sum_{s \in \mathcal{S}}\nabla_\sigma d(s'; \sigma)\cdot \tilde{V}_A(s') =& \sum_{s \in \mathcal{S}}\sum_{s' \in \mathcal{S}} \nabla_\sigma d(s; \sigma) P(s' \mid s; \sigma) \cdot \tilde{V}_A(s') + \\ & \sum_{s \in \mathcal{S}}\sum_{s' \in \mathcal{S}} d(s; \sigma) \nabla_\sigma P(s' \mid s; \sigma)\cdot \tilde{V}_A(s') \, 
\end{align*}
\end{small}
where
$\tilde{V}_A(s')=\sum_{g' \in \mathcal{G}} \sum_{\theta' \in\Theta}  \MU{s'} { \theta' } \cdot \PI{s'}{\theta'}{g'} \cdot V_A(s',g')$ 
$ \text{We have, }
\tilde{r}_A(s) - \tilde{R}_A = \tilde{V}_A(s) - \sum_{s' \in \mathcal{S}} P(s' \mid s) \tilde{V}_A(s') \,$
\begin{small}
\begin{align*}
\text{ Therefore,}
&\sum_{s \in \mathcal{S}}\nabla d(s) \cdot \tilde{r}_A(s) - \sum_{s \in \mathcal{S}}\nabla d(s) \cdot\tilde{R}_A = \\ & \sum_{s \in \mathcal{S}}\nabla d(s) \cdot \tilde{V}_A(s) - \sum_{s \in \mathcal{S}}\nabla d(s) \cdot\sum P(s' \mid s) \tilde{V}_A(s') \, \\
\text{Also, \ } &\tilde{R}_A(\sigma) = \sum_{s \in  \mathcal{S}} d(s; \sigma) \tilde{r}_A(s) \text{, and hence} \\
\nabla_\sigma  \tilde{R}_A= & \sum_{s \in \mathcal{S}}  d(s; \sigma)\cdot \nabla_\sigma \tilde{r}_A(s)  + \sum_{s \in \mathcal{S}}\sum_{s' \in \mathcal{S}} d(s; \sigma) \nabla_\sigma P(s' \mid s; \sigma)\cdot \tilde{V}_A(s')\notag
\end{align*}
\end{small}
\end{proof} 
\vspace{-2em}
\bibliography{mybib}
\bibliographystyle{unsrt}
\end{document}